\documentclass[journal]{IEEEtran}

\usepackage[utf8]{inputenc} 
\usepackage[T1]{fontenc}    
\usepackage{url}            
\usepackage{booktabs}       
\usepackage{amsfonts}       
\usepackage{nicefrac}       
\usepackage{microtype}      

\usepackage{amsthm}
\usepackage{amsmath}
\usepackage{amssymb}
\usepackage[dvipsnames]{xcolor}
\usepackage{enumitem}
\usepackage{mathabx}
\usepackage{enumitem}
\usepackage{algorithm2e}
\usepackage{subfig,float}
\usepackage{comment}
\usepackage{centernot}
\usepackage{tcolorbox}
\usepackage{mathtools}
\usepackage{multicol}
\usepackage{titling}

\usepackage[colorlinks	= true,
			raiselinks	= true,
			linkcolor	= MidnightBlue,
			citecolor	= ForestGreen,
			urlcolor	= Mahogany,
			pdfauthor	= {Armin Eftekhari},
			pdftitle	= {},
			pdfkeywords	= {},   
			pdfsubject	= {},
			plainpages	= false]{hyperref}

\usepackage[left=2cm,right=2cm,top=1.5cm,bottom=1.5cm,includeheadfoot]{geometry}

\usepackage{geometry}

\newtheorem{thm}{Theorem}
\newtheorem{lem}[thm]{Lemma}

\newtheorem{prop}[thm]{Proposition}
\newtheorem{assumption}[thm]{Assumption}

\newtheorem{defn}[thm]{Definition}

\newtheorem{remark}[thm]{Remark}

\newtheorem{target}[thm]{Target}

\def \l{\left}
\def \r{\right}
\def \a{\alpha}
\def \R{\mathbb{R}}

\def \ol{\overline}

\def \g{\gamma}

\def \wt{\widetilde}
\def \s{\sigma}

\def \der{\operatorname{d}\hspace{-1pt}} 
\def \Der{\operatorname{D}\hspace{-1pt}}

\def \A{\mathcal{A}}
\def \P{\mathcal{P}}

\def \F{\text{face}}

\def \P{\mathcal{P}}

\def \dist{\operatorname{dist}}

\def\D{\Delta}
\def\d{\delta}
\def\O{\mathcal{O}}
\def\range{\mathrm{range}}

\def\rank{\mathrm{rank}}
\def\T{\mathcal{T}}
\def\M{\mathcal{M}}

\def \n{\sharp}

\def\int{\mathrm{int}}

\def\trace{\mathrm{tr}}
\def\F{\mathrm{F}}
\def\N{\mathcal{N}}
\def\T{\mathcal{T}}
\def\ker{\mathrm{null}}
\def\Der{\mathrm{D}}

\def\cl{\mathrm{cl}}

\def\c{\color{black}}

\usepackage{tikz}
\newcommand*\circled[1]{\tikz[baseline=(char.base)]{
             \node[shape=circle,draw,inner sep=.8pt] (char) {#1};}}

\definecolor{darkblue}{rgb}{0.0, 0.0, 0.55}

\definecolor{navyblue}{rgb}{0.0, 0.0, 0.5}

\definecolor{darkcerulean}{rgb}{0.03, 0.27, 0.49}

\title{{Over-Parametrized Matrix Factorization\\ in the Presence of\\ Spurious Stationary Points}} 
\author{Armin Eftekhari\thanks{AE is with the Department of Mathematics and Mathematical Statistics at Umea University, Sweden. E-mail: armin.eftekhari@umu.se. AE is indebted to Nicolas Boumal, Konstantinos Zygalakis and Florentin Goyens for their invaluable feedback on an early version of this manuscript, and for related discussions that inspired this work. This work was partially supported by the Wallenberg AI, Autonomous Systems and Software Program (WASP) funded by the Knut and Alice Wallenberg Foundation.  } }

\begin{document}

\maketitle

\begin{abstract} 
M{{\color{black}otivated by {\c the} emerging role {\c of interpolating machines in signal processing and machine learning, this work considers} the computational aspects of over-parametrized 
matrix factorization.}}

{{\c In this context,} the optimization landscape may contain spurious stationary points (SSPs), which are proved to be  full-rank matrices.
The presence of these SSPs means that it is impossible to hope for any global guarantees in over-parametrized matrix factorization. {\c For example,} when initialized at an SSP, the gradient flow will be trapped there forever.

{\c Nevertheless, despite these SSPs,} we establish in this work that the gradient flow of the {\c corresponding} \emph{merit function} converges to a global minimizer, provided that its initialization is rank-deficient and sufficiently close to the \emph{feasible set} {\c of the optimization problem}. We numerically observe that a heuristic discretization of the proposed {\c gradient} flow, inspired by primal-dual algorithms,  is successful when initialized randomly.

{\c Our} result is in sharp contrast with the local refinement methods which require an initialization close to the \emph{optimal set} of the optimization problem. 
More specifically, we successfully avoid} the traps set by the SSPs because {\c the gradient flow} remains rank-deficient at all times, and \emph{not} because there are no SSPs nearby. The latter is the case for the local refinement methods. Moreover, the widely-used restricted isometry property plays no role in our main result.


\end{abstract}

\begin{IEEEkeywords}
r{ank-constrained matrix factorization,  Burer-Monteiro factorization, over-parametrization, interpolation, nonconvex optimization, stationary points}
\end{IEEEkeywords}

\section{Introduction}


{Rank-constrained matrix  factorization from limited data}
is {central to} various applications in signal processing and machine learning~\cite{davenport2016overview}, {and}
{more recently} 
{has} also {served} as a platform to gain theoretical insight into {unexplained} phenomena in deep  neural networks~\cite{gunasekar2017implicit,li2018algorithmic,eftekhari2020implicit}. 

Despite recent strides, we  face key theoretical questions   {in matrix factorization}, 
particularly in {the} emerging areas {that are motivated by the study of neural networks}. The aim of this work is to take a  step towards answering one such question. 



To be concrete, we are interested in the computational {(rather than statistical)} aspects of solving the 
problem
\begin{align}
    \min_{U\in \R^{d\times p}} \, \| U\|_{\F}^2 \,\, \text{subject to}\,\, \A(UU^\top) = b, \,  \|U\|\le \xi,
    \label{eq:bm0}
\end{align}
where $\A:\R^{d\times d}\rightarrow \R^m$ is a linear operator and $b\in \R^{m}$. Above, we have limited the spectral norm of $U$ to $\xi>0$ for technical convenience. By setting $\xi$ sufficiently large, we can practically ignore the constraint $\|U\|\le \xi$ {in~\eqref{eq:bm0}}. 

{Even though our interest in problem~\eqref{eq:bm0} is purely computational,} the {statistical} significance of problem~\eqref{eq:bm0} can be motivated 
{as follows.}
\begin{itemize}[leftmargin=*,wide]
\item {\c The (often nonconvex)} problem~\eqref{eq:bm0} {\c corresponds to}  {the Burer-Monteiro factorization~\cite{burer2005local} of} the {(convex)} nuclear norm minimization problem. {\c This latter convex problem is} at the heart of low-rank matrix sensing and completion, phase retrieval and quadratic sensing, and blind deconvolution, to name a few~\cite{chi2019nonconvex}. 

Compared to its convex analogue, {which has~$d(d+1)/2$ optimization variables},  problem~\eqref{eq:bm0} has~$pd$ variables. {\c Therefore, solving problem~\eqref{eq:bm0}  can} offer computational {gains} {when~$p$ is sufficiently small}~\cite{chi2019nonconvex,boumal2016non}. 

\item {\c Solving} {\c problem~\eqref{eq:bm0} is equivalent to  minimizing the empirical risk of the shallow linear network $x\rightarrow UU^\top x$ with weight decay regularization. The corresponding  training data is} collected in {the vector}~$b$~\cite{arora2018convergence}. {Understanding linear neural networks, such as this one, is a necessary first step in studying (nonlinear) neural networks~\cite{eftekhari2020training}.}

\end{itemize}

{\textbf{Kurzgesagt.} In a nutshell, the  contribution of this work is a gradient flow that provably solves problem~\eqref{eq:bm0} in the over-parametrized regime ($p \gtrsim m/d$), whenever this flow is initialized at a rank-deficient matrix~$U_0$ 
that is sufficiently close to the feasible set of problem~\eqref{eq:bm0}.   

{\c In the over-parametrized regime,} note that we cannot possibly hope for a global scheme for solving problem~\eqref{eq:bm0} because the feasible set of problem~\eqref{eq:bm0} may contain spurious stationary points. {\c These spurious points} can trap the gradient flow, when initialized arbitrarily.

It is then necessary to restrict the initialization~$U_0$ in some way, and our particular choice of a ``sufficiently feasible''~$U_0$ has several precedents within the  nonconvex optimization literature~\cite{eftekhari2020implicit,boumal2016non,sahin2019inexact}. 


We will also describe, albeit less formally, how one can find  an appropriate initialization for the gradient flow.

At the same time, our contribution is fundamentally different from the local refinement methods, within the signal processing literature: {\c These methods} rely on  local strong convexity in a small neighborhood of an isolated global minimizer~\cite[Chapter 5]{chi2019nonconvex}. 

{\c To be specific}, even though our gradient flow requires an initialization~$U_0$  that is sufficiently close to the feasible set of problem~\eqref{eq:bm0}, 
this initialization might be \emph{far} from any global minimizer of problem~\eqref{eq:bm0}. 
Moreover, a small neighborhood of the feasible {\c set of problem~\eqref{eq:bm0}} contains \emph{all} spurious stationary points of problem~\eqref{eq:bm0}, whereas a small neighborhood of {\c an isolated} global minimizer contains \emph{no} other stationary points.


This work appears to be the first to study problem~\eqref{eq:bm0} in the over-parametrized regime ($p\gtrsim m/d$). Therefore, to better motivate our  contributions,
let us first discuss  below the concepts of interpolation and under/over-parametrization, both central to this work. Even though we are solely interested in the computational aspects of solving problem~\eqref{eq:bm0}, whenever possible, we will also motivate the statistical significance of these concepts to enrich the presentation. }

{\subsection{Interpolation}}

{A distinct feature of problem~\eqref{eq:bm0} is its {interpolation} property. That is,} by design, any global minimizer $\ol{U}$ of problem~\eqref{eq:bm0} satisfies the equality constraints  $\A(\ol{U}\,\ol{U}^\top) = b$. {For example, if we interpret~$x\rightarrow \ol{U}\, \ol{U}^\top x$ as a 
shallow linear network~\cite{eftekhari2020training}, this network perfectly interpolates its training data~$b$ and {\c thus} achieves zero training error.} 

This interpolation property has {attracted growing attention within machine learning, in part because modern deep neural networks appear to satisfy this property~\cite{zhang2016understanding}.}
{\c Consequently,} better understanding the {statistical and computational strengths of} interpolating 
learning machines has 
become a major research target~\cite{belkin2018understand,hastie2019surprises}. 
{For instance,} stochastic gradient descent provably exhibits {built-in} variance reduction under the interpolation property~\cite{vaswani2019painless}. 

{To summarize, motivated by the  success of modern neural networks, our interest in problem~\eqref{eq:bm0} partly stems from the need to better understand the computational aspects of interpolating learning machines, such as problem~\eqref{eq:bm0}. }

{Moreover, through the lens of signal processing, the interpolation property of
problem~\eqref{eq:bm0} indicates a subtle but important shift in statistical perspective,  described next.

In the context of low-rank matrix recovery~\cite{davenport2016overview}, let~$b:=\A(X^\natural)+e$ for a hidden model~$X^\natural\in\R^{d\times d}$ and measurement noise vector~$e\ne 0$.  Alternatively, we can {\c also} let~$b:=\A(X^\natural+E)$, where the matrix~$E\ne 0$ represents the model mismatch.   In both cases, the equality constraints~$\A(UU^\top)=b$ in problem~\eqref{eq:bm0} replace the relaxed constraints~$\|\A(UU^\top)-b\|\le \epsilon$,
which is more common in signal processing~\cite{candes2010matrix}. 

This shift in statistical perspective imitates the success of deep neural networks, which often interpolate their training data and yet generalize well on test data~\cite{zhang2016understanding}.
Furthermore, whenever reliable prior knowledge about the probability distributions of the noise~$e$ and  mismatch~$E$ is lacking, it  might be wise to opt for problem~\eqref{eq:bm0}, instead of incorporating the relaxed constraints
with a {\c (possibly)} incorrect choice of norm~$\|\cdot\|$ and~$\epsilon$.}

\subsection{{Under- and Over-Parametrization}}

By counting the number of optimization variables {in problem~\eqref{eq:bm0}}, we can 
distinguish two regimes: {Under- and over-parametrized.}

{The focus of this work is on the over-parametrized regime of problem~\eqref{eq:bm0}, which  has never been studied before, to the best of our knowledge. To better appreciate the computational challenges of the over-parametrized regime, both regimes are 
juxtaposed together below.} 
    
    


\begin{enumerate}[leftmargin=*,wide]
\item ({\sc Under-parametrized regime})
When {$p \lesssim m/d$}, the landscape of problem~\eqref{eq:bm0} is often benign {in the sense that  problem~\eqref{eq:bm0}, even though nonconvex,} has no spurious stationary points~\cite{chi2019nonconvex}. {\c Therefore, in this regime, problem~\eqref{eq:bm0}} can be solved to global optimality {with} a {range} of algorithms, {including the} gradient descent {algorithm and} its perturbed variants~\cite{panageas2016gradient,jin2017escape}.

{More specifically, often in the under-parametrized regime, \emph{every} feasible point of problem~\eqref{eq:bm0} is a global minimizer. That is, the target function~$\|U\|_{\F}^2$ in problem~\eqref{eq:bm0} is  redundant and it  suffices to minimize the feasibility gap~$\|\A(UU^\top)-b\|_2^2$. The latter can be done successfully with, for example, the gradient descent algorithm~\cite{chi2019nonconvex,li2019non,panageas2016gradient}.}

{As a toy example,  in the left panel of Figure~\ref{fig:vis}, the (discretized) gradient flow of the feasibility gap converges to the only feasible point of problem~\eqref{eq:bm0}, which is unique up to a sign. This limit point is also evidently the global minimizer of problem~\eqref{eq:bm0} in this toy example.}



\vspace{2pt}
{From a signal processing perspective, the under-parametrized regime of problem~\eqref{eq:bm0} is particularly well suited for low-rank matrix recovery in the absence of noise:  Let~$b:=\A(X^\natural)$ for a rank-$r$ matrix~$X^\natural\in \R^{d\times d}$, and suppose that~$\mathcal{A}$ is a generic  linear operator. Then, with the choice of~$p=r$, solving problem~\eqref{eq:bm0} uniquely recovers the true model~$X^\natural$, provided that~$m\gtrsim pd = rd$. {\c Here, we have ignored} logarithmic factors for simplicity~\cite{chi2019nonconvex}.     }

\item ({\sc Over-parametrized regime}) 
{In contrast, }when~{$p \gtrsim m/d$},  the feasible set of problem~\eqref{eq:bm0} may contain spurious stationary points which can trap first- or second-order optimization algorithms, {including the} gradient descent. The presence of spurious stationary points means that we \emph{cannot}  hope to globally solve problem~\eqref{eq:bm0}. For example, {\c if we initialize any first-order optimization algorithm} at a spurious stationary point of problem~\eqref{eq:bm0}, {\c the algorithm} will be trapped there forever!

{\c In the over-parametrized regime, we can visualize the computational challenges of solving problem~\eqref{eq:bm0} with a numerical example:}
{In a precise sense,
solving the (constrained) problem~\eqref{eq:bm0} is  equivalent to minimizing its (smooth) 
\emph{merit function}~\cite{nocedal2006numerical}.}
{In the right panel of Figure~\ref{fig:vis}, the (discretized) gradient flow of this {merit} function  converges to a particular feasible point of problem~\eqref{eq:bm0} which is \emph{not} a global minimizer. 
}

{This discouraging observation rules out the possibility of a global scheme for solving problem~\eqref{eq:bm0}, but not all is lost. Indeed, the contribution of this work is a non-global (but also non-local) scheme for solving problem~\eqref{eq:bm0} {\c in the over-parametrized regime.}
}


{Let us next motivate the statistical significance of  the over-parametrized regime. First, recall that modern neural networks are highly over-parametrized~\cite{zhang2016understanding}. In {\c this} sense, the over-parametrized regime of problem~\eqref{eq:bm0}  more faithfully represents the training of a neural network with weight decay, compared to the under-parametrized regime.
}

{Second, from a signal processing perspective, the over-parametrized regime can be motivated as follows.  Consider the problem of low-rank matrix recovery in  the presence of noise: Let~$b:=\A(X^\natural)+e$ for a  hidden rank-$r$ model~$X^\natural$ and measurement noise~$e\in\R^m$.

With this choice of~$\A,b$ and~$p=r$, problem~\eqref{eq:bm0} is not necessarily even feasible because of its interpolation property! That is, with the choice of~$p=r$, we cannot necessarily find a matrix that satisfies the constraints of problem~\eqref{eq:bm0}.  
A larger value of~$p$ is often needed. In particular, in view of the Pataki's lemma~\cite{pataki1998rank,polik2007survey}, it suffices to set~$p\gtrsim \sqrt{2m}$. ({\c Under the mild assumption that}~$m\le 2d^2$, {\c it is easy to verify that} $p\gtrsim \sqrt{2m} \Longrightarrow p \gtrsim m/d$. {\c That is, $p\gtrsim \sqrt{2m}$ falls well within the over-parametrized regime.}) 
{\c This} observation motivates the study of problem~\eqref{eq:bm0} in the over-parametrized regime, {\c from a signal processing viewpoint.} 

On the downside,  the price that we pay for the interpolation property in problem~\eqref{eq:bm0} is a  larger computational burden: Larger~$p$ in the over-parametrized regime means more variables to optimize in problem~\eqref{eq:bm0}. 
As a result, in the context of low-rank matrix recovery, where problem~\eqref{eq:bm0} is the Burer-Monteiro factorization of a convex optimization problem~\cite{chi2019nonconvex}, some \emph{but not necessarily all} of the computational gains of the Burer-Monteiro factorization will be lost inevitably.
}
\end{enumerate}

\subsection{{Scope and} Contributions}

 {
 At last, we are now in position to  fix the scope of this work: We are interested here in the computational (rather than statistical) aspects of solving the (interpolating) machine~\eqref{eq:bm0} in the over-parametrized regime ($ p \gtrsim m/d$). This has never been studied  before, to the best of our knowledge. 
 }



{ Let us summarize  below the contributions of this work  in a simplified and, at times, inaccurate language. In short,} this paper designs a gradient flow that  provably solves  the nonconvex problem~\eqref{eq:bm0} to global optimality in the over-parametrized regime,
{even though the optimization landscape may contain  spurious stationary points.} 
{An informal statement of our main result is presented below.} 

{
\begin{thm}[{\sc Simplified main result, Theorem~\ref{thm:main0}}]\label{thm:informal} Suppose that~$p\gtrsim m/d $. Suppose also that  problem~\eqref{eq:bm0} has a finite  optimal value and a tight convex relaxation. We also consider  the set
\begin{equation}
    \M_b := \l\{ U: \A(U U^\top) = b,\, \|U\|<  \xi \r\},
\end{equation}
which contains all (strictly) feasible matrices in problem~\eqref{eq:bm0}. Suppose that~$\M_b$ is a smooth submanifold of~$\R^{d\times p}$. 

Then there exists a gradient flow that almost surely  converges in limit to a global minimizer of problem~\eqref{eq:bm0}, provided that its initialization~$U_0\in \R^{d\times p}$ is rank-deficient and sufficiently close to the manifold~$\M_b$. 
(Such an initialization matrix  exists under mild assumptions.)
\end{thm}}

{Restricting the initialization is necessary above and we cannot hope to improve Theorem~\ref{thm:informal} and obtain global guarantees. Indeed,  the feasible set of problem~\eqref{eq:bm0} may contain spurious stationary points in the over-parametrized regime which can trap the gradient flow, when initialized arbitrarily, see Figure~\ref{fig:vis}.

At the same time, in Theorem~\ref{thm:informal}, note that the neighborhood of {\c the set}~$\M_b$ 
contains \emph{all} spurious stationary points of problem~\eqref{eq:bm0}. In that sense, Theorem~\ref{thm:informal} is \emph{not} a local convergence result. 

As a practical remark, we later empirically observe  that a random 
initialization~$U_0$ is often a good choice that avoids the worst-case scenario in the right panel of Figure~\ref{fig:vis}. }

{A tight convex relaxation in Theorem~\ref{thm:informal} is a mild assumption {\c which} is met, for example, if we take~$p\gtrsim \sqrt{2m}$~\cite{pataki1998rank,polik2007survey}.  }

{The manifold assumption in Theorem~\ref{thm:informal} is minimal in the sense that it corresponds to the weakest sufficient conditions under which there is in general any hope for the gradient flow  to efficiently find {a} feasible matrix for problem~\eqref{eq:bm0}, i.e., a matrix that satisfies the constraints of~\eqref{eq:bm0}. This assumption has several precedents within the nonconvex optimization literature~\cite{eftekhari2020implicit,sahin2019inexact,boumal2020deterministic}. }

{Theorem~\ref{thm:informal}  radically departs  from the established literature of  rank-constrained matrix factorization. To begin, the restricted isometry property, defined later, which dominates the low-rank matrix recovery literature~\cite{davenport2016overview,chi2019nonconvex}, 
\emph{cannot} hold in the over-parametrized regime of this paper  and 
does not
 appear in Theorem~\ref{thm:informal}. A  review of the related work will follow shortly. 
}


{\textbf{\sc Technical novelty.} The {\c proof of Theorem~\ref{thm:informal} relies on a spectral argument that does} not have any precedents within the matrix factorization literature, to the best of our knowledge. This argument builds on the  observation that, perhaps remarkably, matrix rank does not increase along the gradient flow of the merit function. We have, however,  used {\c simpler} versions of this {\c argument} in earlier works~\cite{eftekhari2020implicit,eftekhari2020training}.  }

\newpage
\onecolumn
\begin{figure}
\begin{center}
    \subfloat
    {\includegraphics[width=8cm,height=5.5cm]{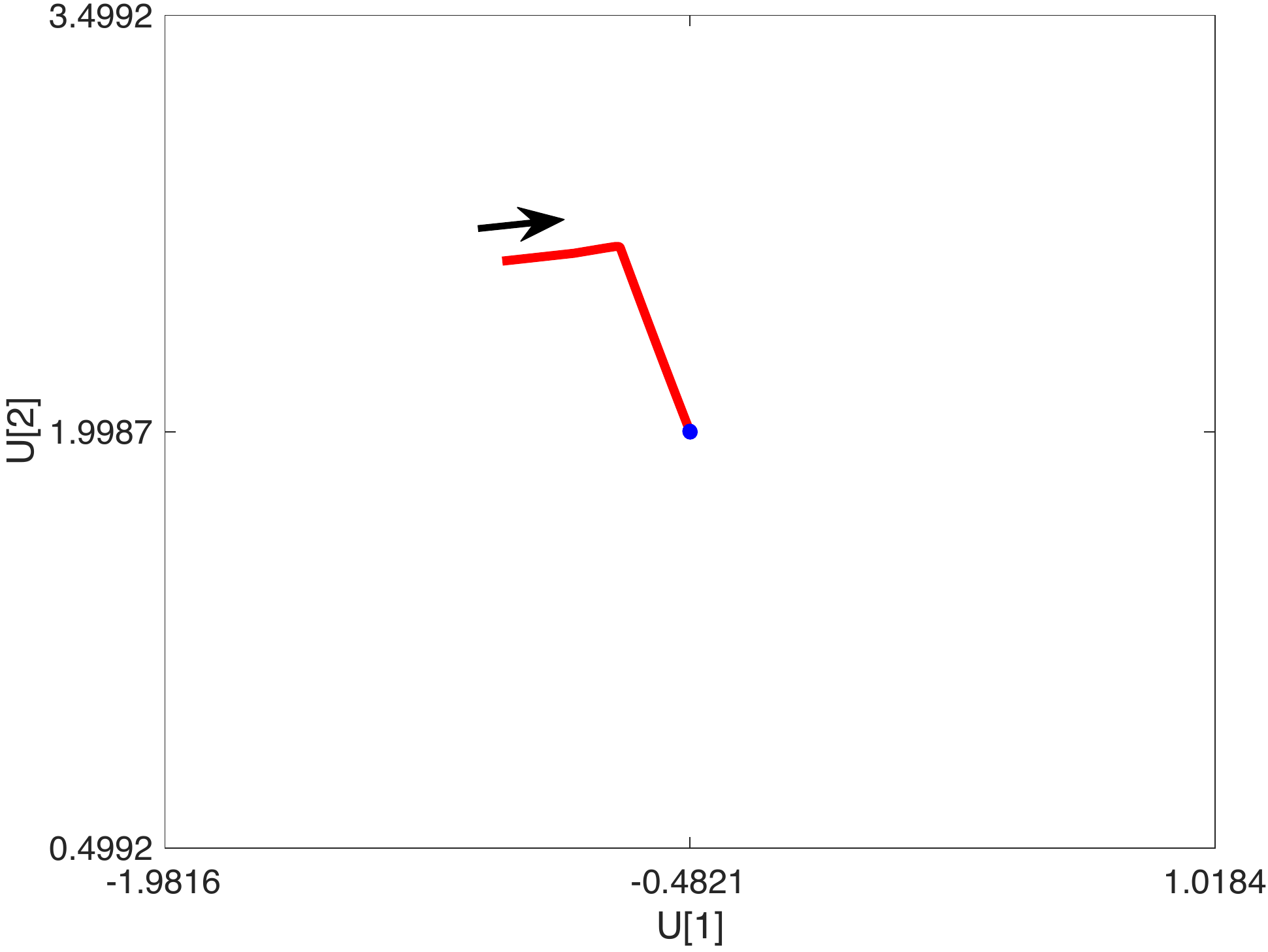}} 
    \hfill
    \subfloat
    {\includegraphics[width=8.5cm,height=5cm]{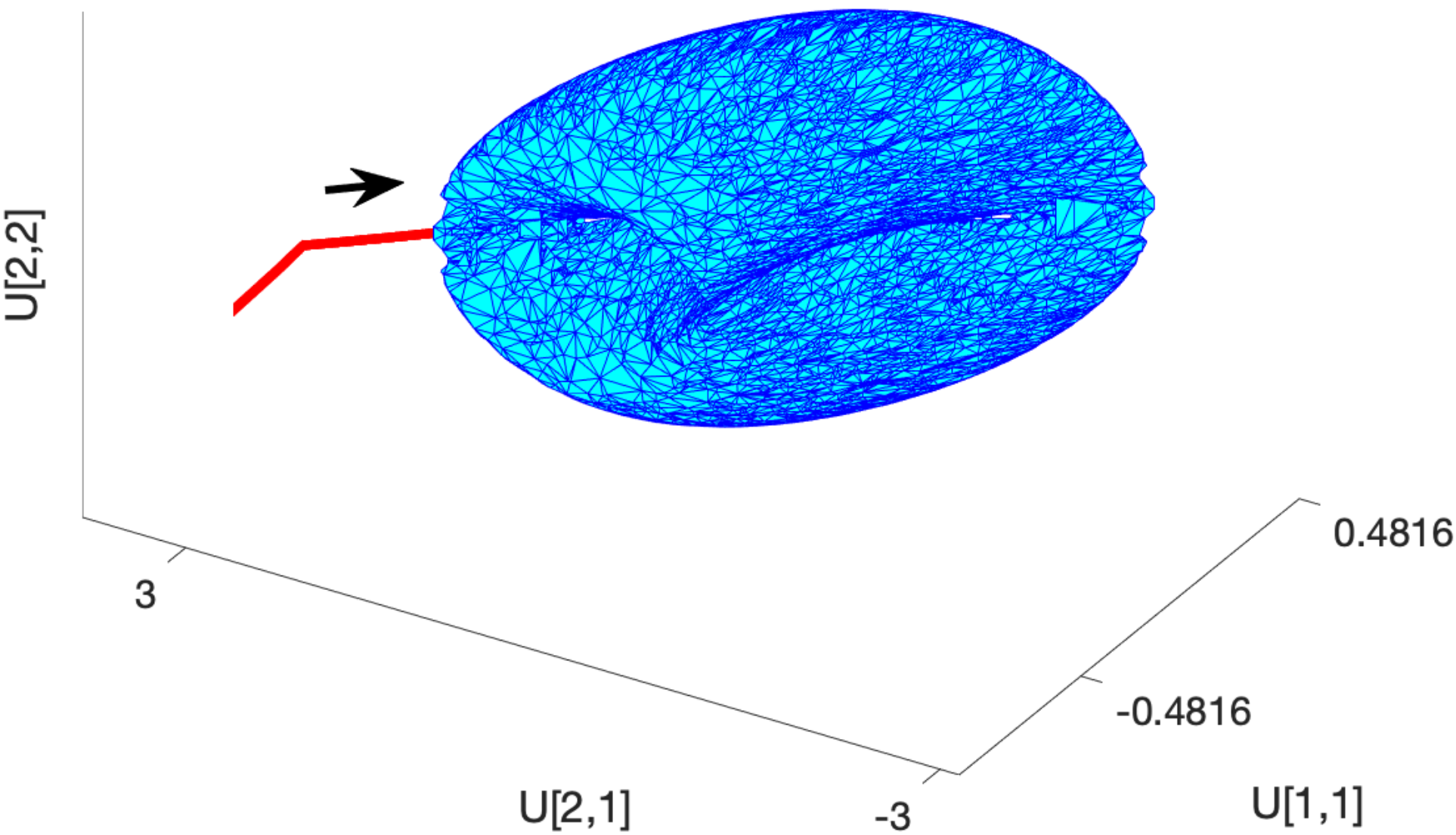}}

    \caption[LoF entry]{{\footnotesize For a toy example, this figure visualizes the differences between the under- and over-parametrized regimes of problem~\eqref{eq:bm0}. In particular, in this toy example,  the feasible set of problem~\eqref{eq:bm0} contains a spurious stationary point in the over-parametrized regime {\c (left panel)} that traps the optimization algorithm, in a sharp contrast with the under-parametrized regime {\c (right panel)}.}
    
    {\footnotesize To be specific, {\c we set}~$d=m=2$,~$\xi=\infty$ and {\c generate}~$\A,b$ {\c randomly, as detailed in the code}. {\c$\A$ and $b$ are} identical in both {\c top and bottom} panels.
    {\c The left panel corresponds to}~$p=1$ whereas, in the right panel, we have~$p=2$. 
    The blue dot in the left panel is the only feasible point of problem~\eqref{eq:bm0}, up to a sign, and  also evidently its global minimizer. {\c On the other hand,} the blue surface in the right panel {\c shows a section} of the feasible set of problem~\eqref{eq:bm0}. The three-dimensional mesh {\c in the right panel was} created with~\cite{mesh}.} 
    
    \noindent {\footnotesize The left panel shows the trajectory of the (discretized) gradient flow of  the feasibility gap~$\|\A(UU^\top)-b\|_2^2$, whereas the right panel plots the  (discretized) gradient flow of the merit function of problem~\eqref{eq:bm0}, to be defined later. {\c Recall that left and right panels correspond to $p=1$ and $p=2$, respectively.} {\c Accordingly,} in the left panel, the  gradient flow is initialized at~$U_0\in \R^{d\times 1}$ whereas, in the right panel, the gradient flow is initialized at~$[U_0\,,\, \begin{psmallmatrix} 0\\   0\end{psmallmatrix}]\in \R^{d\times 2}$. {\c The choice of initialization~$U_0$ is detailed in the code.}
    } 
    
    {\footnotesize In each panel, the limit point of the trajectory is {\c a feasible point and also} a stationary point of problem~\eqref{eq:bm0}.
    In the left panel, the (discretized) gradient flow successfully finds the global minimizer whereas, in the right panel, the (discretized) gradient flow is trapped by a  spurious stationary point.} 
    
    {\footnotesize {\c The stationary point in the right panel is spurious} because the optimal value of problem~\eqref{eq:bm0} equals~$0.24$ in the right panel (obtained by CVX~\cite{cvx,gb08}), whereas the (discretized) gradient flow reaches the suboptimal value of~$4.22>0.24$.  The MATLAB code will be available online.}
    
    }
    \label{fig:vis}
\end{center}
\end{figure}

\begin{figure}
\begin{center}
    \subfloat
    {\includegraphics[width=7cm,height=5cm]{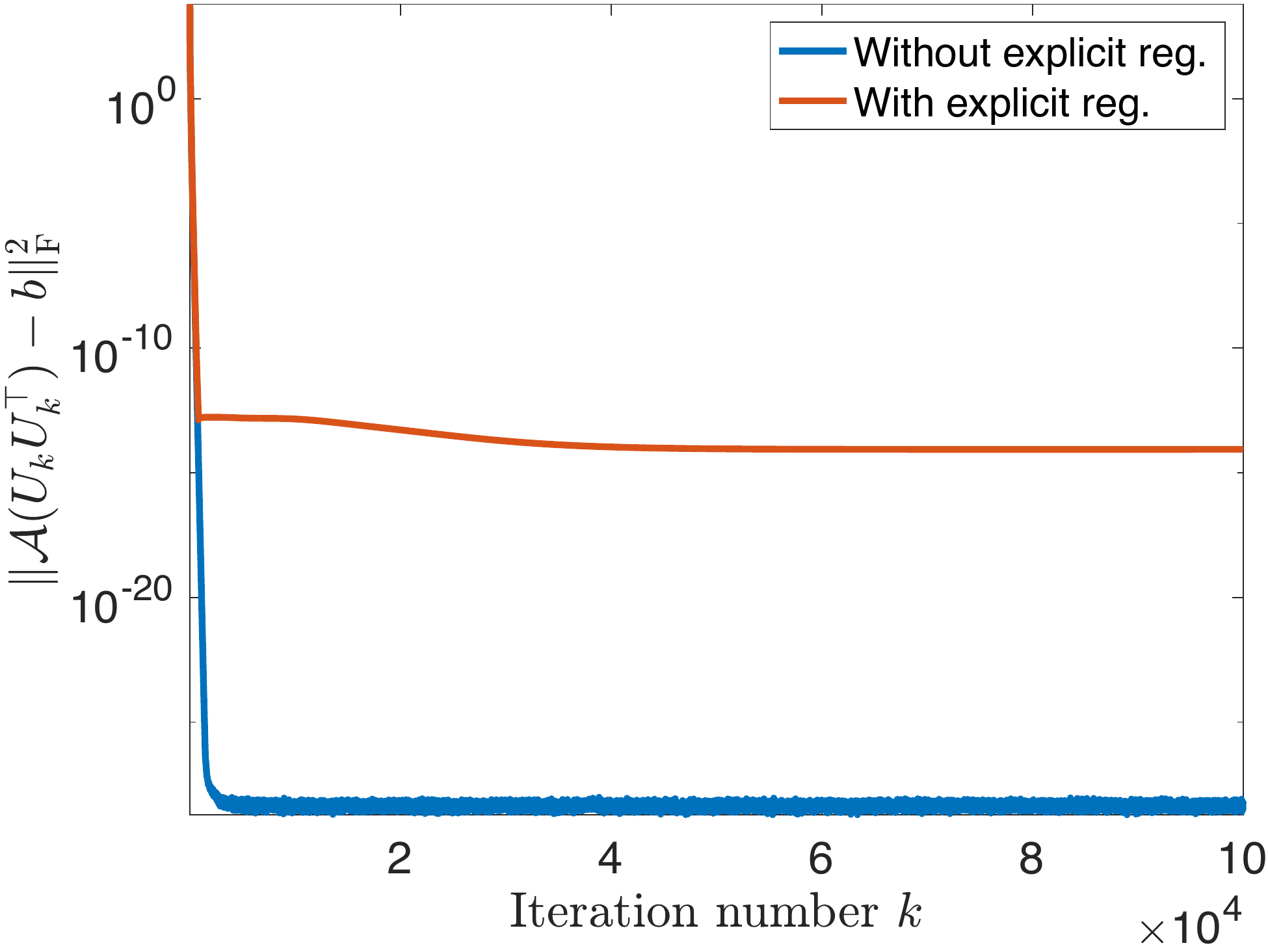}}
    \hfill
    \subfloat
    {\includegraphics[width=7cm,height=5cm]{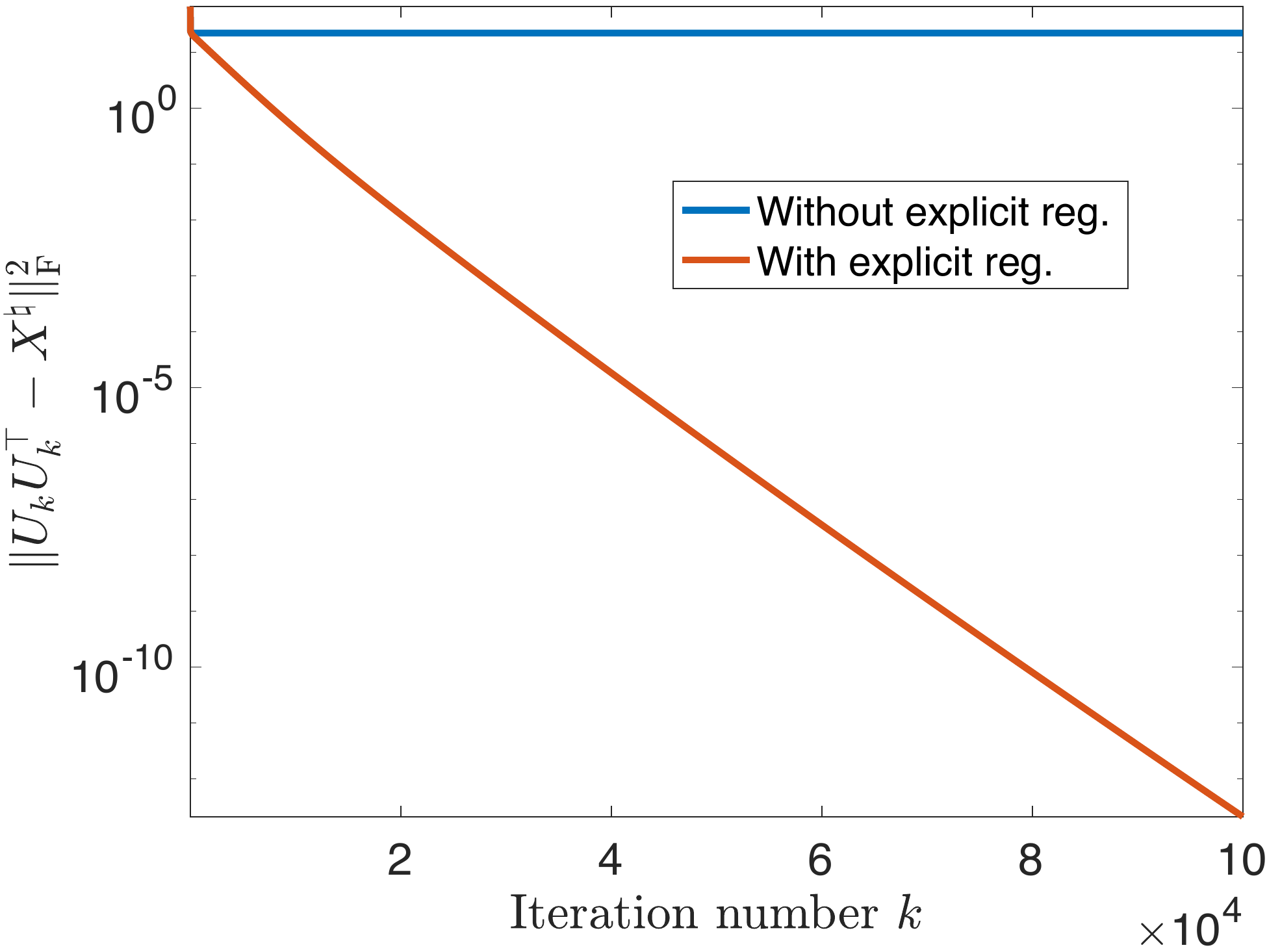}}

    \caption[LoF entry]{ {\footnotesize With a numerical example, this figure illustrates the importance of explicit regularization in problem~\eqref{eq:bm0}. That is, this figure shows that implicit regularization {\c might} fail {\c in general.}  
    }{\footnotesize More specifically, we generated a random rank-$1$ matrix~$X^\natural\in\R^{d\times d}$  with~$d=20$ and~$\|X^\natural\|_{\F}=1$. We also set~$m=4rd=40$ and chose the linear operator~$\A:\R^{d\times d}\rightarrow\R^m$ randomly and then set~$b:=\A(X^\n)$.}

    {\footnotesize For~$p=10$ and with a random initialization,
    we then {\c tried} to recover~$X^\n$ with two approaches:~\circled{1}~First, by following the trajectory of the (discretized) gradient flow of the feasibility gap~$\|\A(UU^\top)-b\|_2^2$ with no explicit regularization of the variable~$U\in\R^{d\times p}$.~\circled{2}~Second, by following the (discretized) gradient flow of the merit function of problem~\eqref{eq:bm0}, which is explicitly regularized by~$\|U\|_{\F}^2$. 
    }

    {\footnotesize {\c The left and right panels show the feasibility gap and recovery error, respectively.}  
    Notably, the first approach, which only relies on the implicit regularization of the gradient flow, fails to recover~$X^\n$. The MATLAB code will be made available with the paper.}
    }
    \label{fig:explicit}
\end{center}
\end{figure}

\twocolumn

\subsection{Approach {and Organization of the Paper}}

Let us {now outline} our approach {to establish Theorem~\ref{thm:informal}}. 
After {a formal setup in Sections~\ref{sec:burer}-\ref{sec:target},}
we first study the nonconvex geometry of problem~\eqref{eq:bm0} in Section~\ref{sec:nonConvexGeom}. We then review the local and global optimality conditions for problem~\eqref{eq:bm0} in Sections~\ref{sec:optimality} and~\ref{sec:globalOpt}, respectively. In particular, {we show that} any rank-deficient second-order stationary point~(SOSP) of problem~\eqref{eq:bm0} is also its global minimizer~\cite{boumal2016non}.  Finally, {in Section~\ref{sec:exact},} we introduce a merit (or exact penalty) function~\cite{nocedal2006numerical}, denoted by $h_\g$, that allows us to reformulate problem~\eqref{eq:bm0} as a smooth and unconstrained optimization problem. 

The main technical contribution of this {\c work} {appears in Section~\ref{sec:mainResults}}.  Theorem~\ref{thm:main0} {therein is the detailed version of the simplified Theorem~\ref{thm:informal} and} can be summarized as follows: Once initialized rank-deficient and sufficiently close to the feasible set of problem~\eqref{eq:bm0}, {the} gradient flow of~$h_\g$ converges to a rank-deficient SOSP of problem~\eqref{eq:bm0} and thus solves problem~\eqref{eq:bm0} to global optimality. {Theorem~\ref{thm:main0} is then followed by several remarks that justify our assumptions.}  

Lastly, {in Section~\ref{sec:compAspc},} we introduce a heuristic discretization of {the} gradient flow {of~$h_\gamma$}. {\c We then show the potential of this algorithm for solving  problem~\eqref{eq:bm0} with a numerical example.} {\c However, we leave} a comprehensive study of this {\c algorithm} as a future research {\c question}.

\section{Related Work}\label{sec:relatedWork}

We are not aware of any prior {literature} on solving problem~\eqref{eq:bm0} in the over-parametrized regime~{($p \gtrsim m/d$)}. Let us instead review a number of closely related works.

If we move the constraints in problem~\eqref{eq:bm0} to its target function (in the form of a penalty term), then~\cite[Theorem~17]{haeffele2015global} provides a high-level procedure for solving the penalized problem that might not converge in polynomial time or might not converge at all; {\c the procedure} might hop from one local minimizer to another, without ever visiting a global minimizer.

An additional concern is that, in nonconvex optimization, often the solutions of the penalized problem approach feasibility only in the limit as the penalty weight grows increasingly large~\cite[Section~17.1]{nocedal2006numerical}.

{Indeed, polynomial time convergence to a global (rather than local) minimizer and the constrained nature of problem~\eqref{eq:bm0}   both will pose significant technical challenges for us.}

We should also mention the literature of implicit regularization, which attempts to recover a planted low-rank matrix~$X^\sharp\in \R^{d\times d}$  by following the gradient flow of  the feasibility gap $\|\A(UU^\top) - b\|_2^2$, {where~$U\in\R^{d\times d}$ and~$b:=\A(X^\sharp)$}. {Often, this gradient flow is}
initialized near the origin~\cite{gunasekar2017implicit,li2018algorithmic,eftekhari2020implicit,geyer2020low}, {and the literature} {\c of implicit regularization}
relies heavily on the restricted isometry property (RIP) of the operator~$\A$~\cite{davenport2016overview}.

{In contrast, we do not assume the RIP in this work and the gradient flow can be initialized anywhere sufficiently close to the feasible set of~\eqref{eq:bm},  rather than near the origin. We also consider the more general case where~$U\in \R^{d\times p}$ with~$p\gtrsim m/d$, rather than the square factorization~$U\in \R^{d\times d}$. 

Perhaps most importantly, problem~\eqref{eq:bm0} is explicitly regularized by~$\|U\|_{\F}^2$. {\c This regularization} directly promotes a low-rank solution, in contrast with the literature of implicit regularization which lacks any explicit regularization. 

In Figure~\ref{fig:explicit}, the importance of explicit regularization is visualized with a numerical  example in the context of low-rank matrix recovery. Implicit regularization fails spectacularly {\c in this figure.} }
Lastly, the technical machinery involved {\c in this work} is 
{fundamentally} different from {\c that in the literature of implicit regularization.}

It is also worth noting that {\c we can adapt the literature of over-parametrized linear networks to our \c situation: This literature}  is concerned with  {\c solving the optimization} problem $\min_U \|\A(UU^\top) - b\|_2^2$ for a particular choice of the {linear} operator~$\A$. {Here, the common approach {\c in}~\cite{arora2018convergence} is often criticized} as ``lazy training''~\cite{chizat2019lazy}, {\c because} it linearizes {\c the map} $U\rightarrow UU^\top$ near a global minimizer {\c of $\min_U \|\A(UU^\top) - b\|_2^2$ and then replaces the nonconvex target function with its local convex approximation.}  

{\c The approach in~\cite{arora2018convergence} thus relies on an initialization near a global minimizer of $\min_U\|\A(UU^\top)-b\|_2^2$, which can be generated randomly when the map $U\rightarrow UU^\top $ is sufficiently over-parametrized. 
We refer to Definition~2 and Theorem~1 in~\cite{arora2018convergence} for more information.} 
Outside of this lazy training regime, we are only aware of the recent work~\cite{eftekhari2020training} {\c for linear neural networks}. 
{\c In contrast, our Theorem~\ref{thm:informal} relies on an initialization near the feasible set (rather than the optimal set) of problem~\eqref{eq:bm0}.} {\c Lazy training is also highly popular in the context of nonlinear neural networks,  e.g.,~\cite{allen2019convergence}.} 

We also mention the \emph{very} over-parametrized regime of~$p\gtrsim d$. {\c In this regime and} {in the context of nonlinear shallow  neural networks,} the landscape of the feasibility gap is known not have any spurious SOSPs~\cite{soltanolkotabi2018theoretical,du2018power}. {\c We can translate this fact to our setup: If}~$p\ge d$, {\c then} it is not difficult to verify from~\eqref{eq:kkt} and Definition~\ref{defn:fosp} that any full-rank SOSP of problem~\eqref{eq:bm0} is also a global minimizer of problem~\eqref{eq:bm0}. Moreover, any rank-deficient SOSP of problem~\eqref{eq:bm0} is a global minimizer of problem~\eqref{eq:bm0} by Proposition~\ref{lem:noSpurious}. 

{\c In analogy with~\cite{soltanolkotabi2018theoretical,du2018power}, we thus obtain the following result:}  When $p\ge d$, problem~\eqref{eq:bm0} does not have any spurious SOSPs. 
This very over-parametrized regime ($p\ge d$) is {evidently} not interesting in {the context of rank-constrained matrix factorization,} and {is not studied in this work.}

Problem~\eqref{eq:bm0} {can also be interpreted as} the  Burer-Monteiro factorization~\cite{burer2005local} of a particular semi-definite program~(SDP). {\c More specifically,} if we replace the target function of problem~\eqref{eq:bm0} with~$\langle C, UU^\top\rangle$ for a generic symmetric matrix~$C$, then the new  {optimization} problem is known not have any spurious SOSPs when~$p \gtrsim \sqrt{2m}$~\cite[Theorem~2]{boumal2016non}.

{\c However,} the randomness of the cost matrix~$C$ {\c means that} \cite[Theorem~2]{boumal2016non} {\c is not applicable to our problem.} Moreover,~{$p\gtrsim \sqrt{2m}\Longrightarrow p\gtrsim m/d$, under the mild assumption that~$m\le 2 d^2$. 
{\c In words}, the over-parametrized regime ($p\gtrsim m/d$) studied  in this paper absorbs~$p\gtrsim \sqrt{2m}$ as a special case.} 


Lastly,  {to read about matrix factorization in} the under-parametrized regime~{($p\lesssim m/d$)}, we wish to refer the reader to the survey~\cite[Section~9]{chi2019nonconvex} and the references therein.

\section{{Nonconvex Optimization Problem}}\label{sec:burer}

For symmetric matrices~$\{A_i\}_{i=1}^m\subset \R^{d\times d}$, {\c let us} consider the linear operator $\A:\R^{d\times d}\rightarrow \R^m$, defined as 
\begin{align}
    \A(X) := [\langle A_1, X\rangle, \cdots, \langle A_m, X\rangle]^\top.
    \label{eq:measOp}
\end{align}
{We {\c can} now formally introduce the nonconvex optimization problem at the center of this work.  }
For $\xi>0$, {integer $p$} and a vector $b\in \R^m$, {{\c we will study} the optimization problem
\begin{align}
    \min_{U\in \R^{d\times p}} \, \| U\|_{\F}^2 \,\, \text{subject to}\,\, \A(UU^\top) = b, \,  \|U\|\le \xi,
    \label{eq:bm}
    \tag{factor}
\end{align}
where~$\|U\|$ denotes the spectral norm of the matrix $U$. {\c Above,} we have  limited the spectral norm {\c to $\xi$} for technical convenience. For a sufficiently large $\xi$, we can practically ignore the constraint~$\|U\|\le \xi$ above.

Even though we are purely interested in the computational aspects of solving \eqref{eq:bm}, the statistical significance of this problem was motivated in the introduction: \eqref{eq:bm} arises as the Burer-Monteiro factorization of the (convex) nuclear norm minimization problem. Alternatively, {\c \eqref{eq:bm}} can be interpreted as regularized empirical risk minimization for training a shallow linear network. } 

{It is worth noting that, in the context of low-rank matrix recovery, the literature also offers some alternatives which can be statistically superior to the learning machine~\eqref{eq:bm} {\c and} its convex relaxation~\cite{tanner2013normalized,bauch2020rank}.  }



\section{{A Convex Relaxation}}\label{sec:relax}

{It is also be helpful to introduce  a convex relaxation of~\eqref{eq:bm}, specified by the semi-definite program 
\begin{align}
    \min_{X\in \R^{d\times d}}  \trace(X) \, \text{subject to}\, \A(X) = b,\,  0 \preccurlyeq X \preccurlyeq \xi^2 I_d,
    \label{eq:sdp}
    \tag{SDP}
\end{align}
where $I_d{\in \R^{d\times d}}$ is the identity matrix, and~$\preccurlyeq$ denotes the positive semi-definite (PSD) pseudo-order.} {\c Indeed, it is easy to obtain~\eqref{eq:sdp} by relaxing the rank restriction in~\eqref{eq:bm}.}

{In the other direction, we {\c can also} interpret \eqref{eq:bm} as the Burer-Monteiro factorization of~\eqref{eq:sdp}.}
{In particular,~\eqref{eq:sdp} explicitly enforces the PSD constraint ($X\succcurlyeq 0$) which was implicit in~\eqref{eq:bm}}. {Moreover, if $p$ is sufficiently small, \eqref{eq:bm} would have fewer optimization variables {\c compared to} \eqref{eq:sdp} and can offer computational savings~\cite{chi2019nonconvex,boumal2016non}. }

{Note that the  restriction to PSD matrices does not reduce the generality of~\eqref{eq:sdp}.} {\c That is, if we choose  $\A':\R^{d_1\times d_2}\rightarrow \R^m$ and $b'\in \R^m$ appropriately, the optimization problem
$$
\min_{Z\in \R^{d_1\times d_2}} \|Z\|_* 
\, \text{subject to}\, \A'(Z) = b',\, \|Z\|\le \xi
$$
reduces to~\eqref{eq:sdp}.
Above, $\|\cdot\|_*$ stands for the nuclear norm. For completeness, we record this observation below.
}
\begin{remark}[{\sc Restriction to PSD matrices}]
The restriction in~\eqref{eq:sdp} to PSD matrices is, for our purposes, without any loss of generality, because a  matrix $Z\in \R^{d_1\times d_2}$ with $\|Z\|\le \xi$ can be mapped or lifted to a PSD matrix $X$ via the map
$$
Z \rightarrow X := \l[ 
\begin{array}{cc}
     \xi^2 I_{d_1} & Z  \\
    Z^\top & I_{d_2}
\end{array}
\r].
$$
\end{remark}

{For future reference, we also}  record that a matrix $\ol{X}\in\R^{d\times d}$ is a (global) minimizer of problem~(\ref{eq:sdp}) 
if there exists a dual certificate $\ol{\lambda}\in \R^m$ such that 
\begin{align}
    & I_d-\A^*(\ol{\lambda}) \succcurlyeq 0, 
    \qquad \l(I_d - \A^*(\ol{\lambda})\r) \ol{X} = 0, \nonumber\\
     & 0 \preccurlyeq  \ol{X}\prec \xi^2 I_d,
    \qquad \A(\ol{X}) = b.
    \label{eq:kkt}
\end{align}
{\c Above,} $\A^*:\R^m\rightarrow\R^{d\times d}$ is the adjoint of the linear operator~$\A$ in~\eqref{eq:measOp}, defined as 
\begin{align}
    \A^*(\lambda) := \sum_{i=1}^m \lambda_i A_i, 
    \label{eq:Aadjoint}
\end{align}
and $\lambda_i$ is the $i^{\text{th}}$ coordinate of the vector $\lambda \in \R^m$.



\section{{Target}}\label{sec:target}

{After  introducing the nonconvex problem~\eqref{eq:bm} and its convex relaxation~\eqref{eq:sdp} in Sections~\ref{sec:burer} and~\ref{sec:relax}, we are now in position to formally state the target of this work. }

{First, let us} record two key assumptions. The first assumption below ensures that the {convex relaxation of~\eqref{eq:bm} is tight},
i.e.,~\eqref{eq:sdp} has \emph{a}  minimizer with rank at most~$p$. 
The second assumption below reflects the limited sampling budget {\c that is} common in {signal processing}. {\c If we regard $x\rightarrow UU^\top x$ as a shallow linear network with weight sharing, the second assumption below} also corresponds to the over-parametrized regime prevalent in neural networks~\cite{eftekhari2020training,oymak2020towards}. 

\begin{assumption}[\sc{Matrix factorization}]\label{assumption:operator}  
\end{assumption}
\begin{enumerate}[label=(\roman*),leftmargin=.6cm] 
\item (Equivalence)\label{assumption:infoLevel}  {\eqref{eq:bm} and~\eqref{eq:sdp}} have the same optimal value, which is assumed to be finite {throughout}.  
\item (Over-parametrized)\label{assumption:badRegime}  $ m < pd$
 if $p< d/2$ and $m< d(d+1)/2$ otherwise. 
\end{enumerate}

{\c Assumption~\ref{assumption:operator}.\ref{assumption:infoLevel} is not too restrictive:}
{For example,} {\c consider} a planted matrix~$X^\sharp$ {\c such that $\rank(X^\n)\le p$} {and {\c set}~$b:=\A(X^\sharp)$}. {\c Then it is not difficult to verify that }Assumption~\ref{assumption:operator}.\ref{assumption:infoLevel} is fulfilled {\c when the linear operator}~{$\A$ {\c satisfies the RIP of order~$2r$}~\cite{davenport2016overview}.} {\c As another example, for an arbitrary linear operator $\A:\R^{d\times d}\rightarrow\R^m$, we can use the  Pataki's lemma to  verify that Assumption~\ref{assumption:operator}.\ref{assumption:infoLevel} holds whenever $p\gtrsim \sqrt{2m}$~\cite{pataki1998rank}\cite[Theorem 6.1]{polik2007survey}. As mentioned in Section~\ref{sec:relatedWork}, 
$p\gtrsim \sqrt{2m}$ does not violate Assumption~\ref{assumption:operator}.\ref{assumption:badRegime}, under the mild assumption that $m=O(d^2)$.}
{Let us now}
state our target.

\begin{tcolorbox}[width=1\columnwidth,colback={gray!40}]
\begin{target}\label{o:target} Suppose that Assumption~\ref{assumption:operator} is fulfilled. Our aim  is to solve {\eqref{eq:bm}} to global optimality. 
\end{target}
\end{tcolorbox}


{\c We will soon see} that Target~\ref{o:target} is out of the reach {\c for} the current literature of matrix factorization. 
The main contribution of this paper is {\c to achieve} Target~\ref{o:target}: 
Sections~\ref{sec:nonConvexGeom}-\ref{sec:exact} will help us develop the necessary tools to achieve  Target~\ref{o:target} and our main result is then presented in Section~\ref{sec:mainResults}.

{\c For technical convenience}, let us {\c also} define the functions $f:\R^{d\times p}\rightarrow\R$, $g:\R^{d\times p}\rightarrow\R^m$ and $G:\R^{d\times p}\rightarrow\R$ as 
\begin{align}
    & f(U) := \frac{1}{2}\|U\|_{\F}^2 ,\qquad g(U):= \frac{1}{2}( A(UU^\top) - b), \nonumber\\
    & G(U) := \frac{1}{2}\|g(U)\|_2^2.
    \label{eq:fgDefn}
\end{align}
{\c Note that} \eqref{eq:bm} has the same minimizers as {the {\c optimization} problem} $\min_U \, f(U)$ subject to $g(U)=0$ and $\|U\|\le \xi$. We can interpret~$G$ in~\eqref{eq:fgDefn} as the (scaled) feasibility gap associated with \eqref{eq:bm}. The new functions $f,g,G$ will frequently appear in our analysis {throughout}.

\section{Nonconvex Geometry}\label{sec:nonConvexGeom}

In Section~\ref{sec:burer}, we introduced \eqref{eq:bm} as the nonconvex optimization program at the {center} of this work.
{Now,} as {our} first step towards Target~\ref{o:target}, we study in this section the geometry of the feasible set of \eqref{eq:bm}.

{In Figure~\ref{fig:vis}, the blue dot and the blue surface visualize  parts of the feasible set of~\eqref{eq:bm} for two toy examples.}
For brevity, we will  denote the interior of the feasible set of \eqref{eq:bm} by $\M_b$. That is, 
\begin{align}
    {\M_b} & := \l\{ U: \A(U U^\top) = b,\, \|U\|<  \xi \r\}
    \subset \R^{d\times p}
    \nonumber\\
    & = \{U: g(U) = 0,\, \|U\|<   \xi \}.\quad  \text{(see \eqref{eq:fgDefn})}
    \label{eq:manifold}
    \tag{manifold}
\end{align}
In particular, note that \eqref{eq:bm} has the same minimizers as {the optimization problem} $\min_U\, f(U) $ subject to $U\in \cl(\M_b)$, where $\cl(\M_b)$ denotes the closure of~$\M_b$.  

After recalling~\eqref{eq:fgDefn}, we also observe that
\eqref{eq:bm} finds the distance from the origin to the set~$\M_b$.
 {As a minor transgression, in our qualitative discussions, we will occasionally refer to~$\M_b$ as the feasible set of~\eqref{eq:bm}, even thought $\M_b$ is the interior of that feasible set, strictly speaking.}

In the rest of this section, we will study the geometry of the set $\M_b$.
For future {reference}, let us record that the (total) derivative of $g$ at $U$ is the linear operator  $\Der g(U):\R^{d\times p} \rightarrow\R^m$, {specified} as
\begin{align}
    \Der g(U)[\D] := \A(\D U^\top),
    \label{eq:defnC}
\end{align}
and its adjoint $(\Der g(U))^* : \R^m \rightarrow \R^{d\times p}$ is defined as
\begin{align}
    (\Der g(U))^* [\d] := \A^*(\d)\cdot U = \sum_{i=1}^m \d_i A_i U,
        \label{eq:adjO}
\end{align}
where $\A^*$ was defined in~\eqref{eq:Aadjoint} and~$\d_i$ is the $i^{\text{th}}$ entry of the vector $\d\in \R^m$. To obtain~\eqref{eq:defnC}, we {leveraged} the fact that~$\{A_i\}_{i=1}^m$ are symmetric matrices.  
Also for later reference, the neighborhood of the set $\M_b$ is defined as follows.
\begin{defn}[{\sc Neighborhood of $\M_b$}]\label{defn:neighb}
The neighborhood of radius $\rho>0$ of 
the set $\M_b$ in~\eqref{eq:manifold} is the (closed) set 
\begin{align}
 \{U: \dist(U,\M_b) \le \rho\} \subset \R^{d\times p},
 \label{eq:neighborS}
 \tag{neighborhood}
\end{align}
where 
\begin{align}
\dist(U,\M_b)&  := \inf\l\{ \|U-U'\|_\F: U'\in \M_b \r\} \nonumber\\
& = \min\l\{\|U-U'\|_\F:  U'\in \cl(\M_b) \r\} \nonumber\\
& = \dist(U,\cl(\M_b ))
\label{eq:distDefn}
\tag{metric}
\end{align}
is the distance from {the matrix} $U$ to the set $\M_b$ and its closure.  {For brevity, we often use the term $\rho$-neighborhood of $\M_b$ throughout. We will also} often say~$U$ is sufficiently close to~$\M_b$ to mean that $\dist(U,\M_b)$ is sufficiently small. 
\end{defn}
{As claimed above, the $\rho$-neighborhood of $\M_b$
is indeed a closed set} because~$\dist(U,\M_b)$ is a continuous function of~$U$ and $[0,\rho]$ is a closed interval.
A short remark follows next to justify our choice of metric in Definition~\ref{defn:neighb}.
\begin{remark}[{\sc Invariance of the metric}] 
The metric in Definition~\ref{defn:neighb} is  invariant under rotation from right. That is,  for any~$U\in \R^{d\times p}$ and~$R\in \O_p$, it holds that 
\begin{align}
    \dist(U R ,\M_b) & = \inf_{U'\in \M_b} \|U R - U'\|_\F \nonumber\\
    & = \inf_{U'\in \M_b} \|U R - U'R \|_\F \nonumber\\
    & = \inf_{U'\in \M_b} \|U - U'\|_\F \nonumber\\
    & = \dist(U ,\M_b).
\end{align}
Above,  $\O_p =\{R': R'^\top R' = I_p\}\subset \R^{p\times p}$ is the orthogonal group and~$I_p{\in \R^{p\times p}}$ is the identity matrix. Moreover, the second line above holds because  $U'\in \M_b$ if and only if $U'R\in \M_b$, see~\eqref{eq:manifold}. The third line above holds by the rotational invariance of the Frobenius norm.  

\end{remark}

We are now in position to state the {central}
 assumption of {this} work, {which also subsumes the earlier Assumption~\ref{assumption:operator}.}
{Assumption~\ref{assumption:key} is similar to~\cite[Assumption~1.1]{boumal2020deterministic} or~\cite{eftekhari2020implicit}. As clarified later, Assumption~\ref{assumption:key} enables us to efficiently find feasible points of~\eqref{eq:bm} which is evidently necessary for solving~\eqref{eq:bm}.} 


\begin{assumption}[{\sc Manifold}] \label{assumption:key}
Assumption~\ref{assumption:operator} is fulfilled and it also holds that 
\begin{align}
    \rank(\Der g(U)) = m, 
    \label{eq:fullRankAll}
\end{align}
for every {\c matrix}~$U$ 
in a sufficiently small neighborhood of the  set~$\M_b$, i.e., the matrices~$\{A_i U\}_{i=1}^m$ are linearly independent for every {\c matrix}~$U$ that is sufficiently close to the set~$\M_b$. 
\end{assumption}

Under Assumption~\ref{assumption:key}, {note that}  $\M_b$ is a closed embedded submanifold of $\R^{d\times p}$ of co-dimension $m$, see~\cite[Corollary~5.24]{lee2013smooth}. 
 {As suggested by Assumption~\ref{assumption:key}, we will frequently use the qualifiers ``sufficiently small'' and ``sufficiently close'' throughout this work. This decision is justified below. }

\begin{remark}[{\sc Sufficiently small / close}]\label{rem:suffClose1}
A conservative lower bound {for} the radius of the neighborhood in Assumption~\ref{assumption:key} is given by \cite[Proposition~13]{eftekhari2020implicit}. 

However, the lower bound in~\cite{eftekhari2020implicit} is  of little practical value because it involves certain geometric attributes of the set~$\M_b$ in~\eqref{eq:manifold} which are difficult to estimate. 

In view of this {practical} limitation and also to avoid any unnecessary clutter, we
 opted not to precisely specify the neighborhood size in Assumption~\ref{assumption:key}. 
 
 Instead, Assumption~\ref{assumption:key} and most of our results are stated for a ``sufficiently small'' neighborhood of the set $\M_b$ with respect to~\eqref{eq:distDefn}. We will revisit this issue later on. Note also that the set $\M_b$ and its neighborhood are often large sets, i.e., our results are not local. 
 

\end{remark}

Before closing this section, we collect below the geometric properties of~$\M_b$:
It is not difficult to verify that the normal space of the smooth manifold~$\M_b {\subset \R^{d\times p}}$ at the matrix $U\in \M_b$ {can be identified with} the linear subspace 
\begin{align}
\N_U {\M_b} & := \range((\Der g(U))^* ) \nonumber\\
& = \mathrm{range}\l(\{ A_i U \}_{i=1}^m \r) \subset \R^{d\times p}.
\qquad \text{(see \eqref{eq:adjO})}
\label{eq:normalSpace}
\end{align}
The tangent space of $\M_b$ at $U\in \M_b$ immediately follows 
from the fundamental theorem of linear algebra. That is, 
\begin{align}
    \T_U {\M_b} := \ker(\Der g(U)),
    \qquad \text{(see \eqref{eq:normalSpace})}
    \label{eq:tangentSpace}
\end{align}
where $\ker(\Der g(U))$ denotes the null space (or kernel) of the linear operator $\Der g(U)$ in~\eqref{eq:defnC}. 

In view of~\eqref{eq:normalSpace}, the orthogonal projection onto the tangent space at $U\in \M_b$ is the linear operator $\P_{\T_U{\M_b}}:\R^{d\times p}\rightarrow \T_U{\M_b}$ that maps the matrix $\Delta\in \R^{d\times p}$ to 
\begin{align}
    \P_{\T_U{\M_b}}(\D)&  := \D - \l((\Der g(U))^* \circ ((\Der g(U))^*)^\dagger\r) [\D] \nonumber\\
    &  = \D - (\Der g(U))^*  \l[((\Der g(U))^*)^\dagger [\D] \r] \nonumber\\
    & =: \D - \A^*(\lambda(U,\D)) \cdot U,
    \quad \text{(see \eqref{eq:adjO})}
    \label{eq:projTangent}
\end{align}
where, for brevity, above we defined 
\begin{align}
    \lambda(U,\D) := ((\Der g(U))^*)^\dagger [\D].
        \label{eq:lambdaDefn}
        \tag{multipliers}
\end{align}
In~\eqref{eq:projTangent}, $\circ$ {represents} the composition of two operators and~$\dagger$ stands for the Moore-Penrose pseudo-inverse. The vector~$\lambda(U,U)$ in~\eqref{eq:lambdaDefn}  plays a key role in the ensuing arguments and we close this section with a technical observation about this quantity. 

\begin{lem}[{\sc Derivative of Lagrange multipliers}]\label{lem:derLambda}
Suppose that  Assumption~\ref{assumption:operator}.\ref{assumption:badRegime} holds. For a matrix $U\in \R^{d\times p}$, $\lambda(U,U)$ in~\eqref{eq:lambdaDefn} has the closed-form expression
 \begin{align}
     & \lambda(U,U) = K(U)^\dagger \A(UU^\top), \nonumber\\
     &  K(U):= [\langle A_i U, A_j U\rangle]_{i,j=1}^m \in \R^{m\times m}.
     \label{eq:lambdaExplicitLem}
 \end{align}
 Suppose now  that Assumption~\ref{assumption:key} holds. Then, $\lambda(U,U)$ is an analytic function of $U$ on a sufficiently small neighborhood of~$\M_b$ in~\eqref{eq:manifold}.
 {For completeness}, the directional derivative of $\lambda(U,U)$ with respect to~$U$ is given explicitly by~\eqref{eq:lambdaDerFinalLemma} in the supplementary material.
%
\end{lem}

\section{Local  Optimality}\label{sec:optimality}

In Section~\ref{sec:nonConvexGeom}, we studied the geometry of the feasible set of~\eqref{eq:bm}. See $\M_b$ in~\eqref{eq:manifold}. As our next step towards Target~\ref{o:target},  we will review in this section the sufficient conditions for local optimality in \eqref{eq:bm}.

To begin, recall that the manifold gradient of $f$ in~\eqref{eq:fgDefn} at the matrix $U\in \M_b$ is {defined as} the projection of the (Euclidean) gradient of $f$ onto the tangent space of the manifold at~$U$~\cite{boumal2020introduction}. That is, 
\begin{align}
     \nabla_{\M_b} f(U) & := \P_{\T_U{\M_b}}(\nabla f(U)) \nonumber\\
    & = \P_{\T_U {\M_b}}(U) \qquad \text{(see \eqref{eq:fgDefn})} \nonumber\\
    & = (I_d - \A^*(\lambda(U,U)) \cdot U \in \R^{d\times p},
    \label{eq:manifoldGrad}
\end{align}
where the last line above follows from~\eqref{eq:projTangent}.
Above, $\nabla$ stands for (Euclidean) gradient and $\nabla_{\M_b}$ denotes the manifold gradient for $\M_b$.

For \eqref{eq:bm}, the matrix $\ol{U}\in \M_b$ is a first-order stationary point if the manifold gradient of $f$ vanishes at $\ol{U}$. More specifically, after recalling the notation in~\eqref{eq:fgDefn} and~\eqref{eq:lambdaDefn}, we record the following. 

\begin{defn}[{\sc FOSP}]\label{defn:fosp}
Consider a matrix $\ol{U}\in \R^{d\times p}$ such that $\|\ol{U}\|< \xi$. This matrix $\ol{U}$  is a first-order stationary point~(FOSP) of  \eqref{eq:bm} if 
\begin{align}
    & g(\ol{U}) = \frac{1}{2} (A(\ol{U} \cdot \ol{U}^\top) - b) =0,\nonumber \\
    & \nabla_{\M_b} f(\ol{U}) = (I_d - \A^*(\lambda(\ol{U},\ol{U}))\cdot \ol{U}  = 0. 
    \label{eq:fosp}
\end{align}
\end{defn}

Moving on to second-order optimality, we denote the manifold Hessian of $f$ in~\eqref{eq:fgDefn} at the matrix $U\in \M_b$ with the bilinear map $\nabla^2_{\M_b} f(U):\T_U{\M_b}\times \T_U{\M_b} \rightarrow \R $~\cite{boumal2020introduction}. This bilinear operator maps  $[\Delta,\Delta]\in \R^{d\times p}\times \R^{d\times p}$ to the scalar
\begin{align}
    & \nabla_{\M_b}^2 f(U)[\D,\D] \nonumber\\
    & 
    := \l( \nabla^2 f(U) - \sum_{i=1}^m \lambda_i(U,U) \cdot \nabla^2 g_i(U) \r)     [\D,\D]    \label{eq:manifoldHessian} \\
    & = 
    \|\D\|_\F^2 - \l\langle   \A^*(\lambda(U,U)),   \D \D^\top  \r\rangle  , \qquad \text{(see \eqref{eq:fgDefn},\eqref{eq:adjO})} \nonumber
\end{align}
for $\D\in \T_U\M_b$. Above, $\nabla^2$ stands for (Euclidean) Hessian. Also, $\lambda_i(U,U)$ and $g_i(U)$ are the $i^{\text{th}}$ coordinates of the vectors $\lambda(U,U)$ and $g(U)$, respectively.

For \eqref{eq:bm}, a second-order stationary point has a PSD manifold Hessian, as detailed below. 
\begin{defn}[{\sc SOSP}]\label{defn:sosp}
Consider a matrix $\ol{U}\in \R^{d\times p}$ such that $\|\ol{U}\|<\xi$. This matrix $\ol{U}$ is a second-order stationary point (SOSP) of \eqref{eq:bm} if, in addition to~\eqref{eq:fosp}, the manifold Hessian $\nabla_{\M_b}^2 f(\ol{U})$ in~\eqref{eq:manifoldHessian} is a PSD linear operator:  
\begin{align}
    \l\langle \D,  \l( I_d - \A^*(\lambda(\ol{U},\ol{U})) \r)  \D \r\rangle \ge 0,
    \quad \text{if } \D\in \T_{\ol{U}}{\M_b}. 
    \label{eq:sosp}
\end{align}
Above, $\T_{\ol{U}}\M_b$ is the tangent space of the manifold $\M_b$ at the matrix $\ol{U}$, see~\eqref{eq:tangentSpace}.
\end{defn}


\section{Global Optimality}\label{sec:globalOpt}

In Section~\ref{sec:optimality}, we reviewed the local optimality conditions for \eqref{eq:bm}, e.g., see Definition~\ref{defn:sosp} of an SOSP. In general, however, not every SOSP of \eqref{eq:bm} is a global minimizer. That is, some SOSPs might be local minimizers or non-strict saddle points of~\eqref{eq:bm}. {We label such points as the spurious SOSPs of the optimization  problem~\eqref{eq:bm}.}

\begin{defn}[{\sc Spurious SOSP}] An SOSP of \eqref{eq:bm} is spurious if it is not a global minimizer of \eqref{eq:bm}.
\end{defn}


{In the under-parametrized regime ($p\lesssim m/d$),} 
{a generic linear operator~$\A$ in~\eqref{eq:measOp} often satisfies the RIP of order~$2p$~\cite{davenport2016overview}.}
 
 Based on this observation, it can be shown that every feasible point of~\eqref{eq:bm} is also a global minimizer of~\eqref{eq:bm}. In other words,~\eqref{eq:bm}  does not have any spurious~SOSPs in the under-parametrized regime for a generic operator $\A$. 

In {this under-parametrized}  regime, even though nonconvex,~\eqref{eq:bm} {can be solved to global optimality by}
a variety of first- or second-order optimization algorithms, {including the} gradient descent~\cite{panageas2016gradient}. 

{While the focus of this work is on the over-parametrized regime~($p\gtrsim m/d$), the under-parametrized regime of~\eqref{eq:bm} is also reviewed in the arXiv version of this paper {\c for completeness}, see also~\cite{chi2019nonconvex}.  }

{Unlike the under-parametrized regime and its benign optimization landscape, it is far more difficult in general to solve~\eqref{eq:bm} in the over-parametrized regime:  }


\begin{remark}[{\sc Spurious SOSPs}]\label{q:q}
Under Assumption~\ref{assumption:operator}.\ref{assumption:badRegime}, we see by counting the degrees of freedom that the linear operator $\A$ in~\eqref{eq:measOp} \emph{cannot} satisfy the RIP of order $2p$, unlike {the under-parametrized regime}. 

\eqref{eq:bm} might therefore have spurious SOSPs which could trap a first- or second-order optimization algorithm, {including the} gradient descent. {A toy example of this pathological situation appeared earlier in Figure~\ref{fig:vis} (right panel).}
 
{This discouraging observation rules out the possibility of a global scheme for solving problem~\eqref{eq:bm0}. For instance, initialized at a spurious stationary point, gradient descent remains there forever.

Fortunately, not all is lost. In the remainder of this work, we will devise a gradient flow that solves~\eqref{eq:bm} to global optimality and achieves Target~\ref{o:target}, when initialized within a ``capture neighborhood'' of the feasible set of~\eqref{eq:bm}. 

Lastly, note  that the feasible set and its neighborhood are often large sets, i.e., our results are not local.  }

\end{remark}

{Before closing this section,} let us recall a sufficient condition for global optimality of an~SOSP, see~\cite{boumal2016non,haeffele2015global}. 
\begin{prop}[{\sc Global optimality}]\label{lem:noSpurious}
    Any rank-deficient SOSP~$\ol{U}$ of \eqref{eq:bm} with~$\|\ol{U}\|<\xi$ is also a global minimizer of~\eqref{eq:bm}. Moreover,~$\ol{U}\,  \ol{U}^\top$ is a (global) minimizer of~\eqref{eq:sdp}. By rank-deficient, we mean that~$\ol{U}$ is a singular matrix.  
\end{prop}



    


\section{{Merit} Function }\label{sec:exact}

Our next step towards Target~\ref{o:target} is as follows:
{While}~\eqref{eq:bm} is a constrained  optimization program,
we introduce in this section a merit (or exact penalty) function that allows us to reformulate~\eqref{eq:bm} as a smooth (and unconstrained) optimization program.

{The main contribution of this paper, as we will see shortly, is using the gradient flow of this merit function to solve~\eqref{eq:bm}.}

To begin, for $\g>0$, let $L_\g:\R^{d\times p}\times \R^m \rightarrow \R$ denote the (scaled) augmented Lagrangian~\cite{nocedal2006numerical} associated with \eqref{eq:bm}, defined as 
\begin{align}
    L_\g(U,\lambda') & := f(U) - \langle g(U), \lambda' \rangle + \frac{\g}{2}\| g(U)\|_2^2,
        \label{eq:al}
        \tag{AL}
\end{align}
where $f$ and $g$ were specified in~\eqref{eq:fgDefn}.
The augmented Lagrangian has two remarkable properties that are listed in the next proposition, inspired by~\cite[Theorem 17.5]{nocedal2006numerical}.

{Loosely speaking, Proposition~\ref{prop:lyapunovPre} posits that the augmented Lagrangian encodes the optimality criteria of~\eqref{eq:bm}. The result below also allows us to interpret~$\lambda(\ol{U},\ol{U})$ in~\eqref{eq:lambdaDefn} as the (dual) optimal Lagrange multipliers for a feasible matrix~$\ol{U}\in \M_b$ in~\eqref{eq:manifold}. }

\begin{prop}[{\sc Augmented Lagrangian}]\label{prop:lyapunovPre}   Suppose that  Assumption~\ref{assumption:key} holds. Consider  a matrix $\ol{U}\in \R^{d\times p}$ that is sufficiently close to $\M_b$ and satisfies  $\|\ol{U}\|< \xi$. 
For~$\g>0$, the following statements are true:
\begin{enumerate}[label=(\roman*),leftmargin=.6cm]
\item $\nabla_1 L_\g(\ol{U},\lambda(\ol{U},\ol{U})) = 0$ implies that $\ol{U}$ is an FOSP of \eqref{eq:bm} and, in particular, $\ol{U}\in \M_b$.  

\item  If, in addition,  $\nabla^2_1 L_\g(\ol{U},\lambda(\ol{U},\ol{U}))[\D,\D] \succcurlyeq 0$ for every tangent direction $\D\in \T_{\ol{U}} \M_b$ in~\eqref{eq:tangentSpace}, then $\ol{U}$ is also an SOSP of \eqref{eq:bm}. 
\end{enumerate}
Above, $\nabla_1 L_\g$ and $\nabla^2_1 L_\g$ are the (Euclidean) gradient and Hessian of $L_\g$ with respect to its first argument, respectively. {For example,~$\nabla_1 L_\gamma(\ol{U},\lambda(\ol{U},\ol{U}))$ is the gradient of~$L_\gamma(\cdot,\cdot)$ with respect to its first argument and evaluated at the pair~$(\ol{U},\lambda(\ol{U},\ol{U}))$.}  
\end{prop}

Inspired by the properties of the augmented Lagrangian, let us now consider the function $h_\g:\R^{d\times p}\rightarrow \R$, defined as
\begin{align}
    h_\g(U) := L_\g(U,\lambda(U,U)), \,\,\,\,\, 
    \label{eq:hGamma}
    \tag{merit}
\end{align}
also known as the Fletcher's augmented Lagrangian~\cite{nocedal2006numerical}. 
{Before we uncover the key property of~$h_\gamma$ and its namesake,} let us {first} record that~$h_\g$ is an analytic function.

\begin{lem}[{\sc Derivative of $h_\gamma$}]\label{lem:derhGamma}  Suppose that Assumption~\ref{assumption:key} holds. Then,  $h_\g(U)$ in~\eqref{eq:hGamma}  is an analytic function of $U$ on a  sufficiently small neighborhood of $\M_b$ in~\eqref{eq:manifold}.
Its derivative is specified as 
\begin{align}
    \nabla h_\g(U)&  = (I_d - \A^*( \lambda(U,U) )) U +\frac{\g}{2} \A^*\l(  \A(UU^\top) - b \r) U \nonumber\\
    & -\frac{1}{2}  (\Der \lambda(U,U))^*[\A(UU^\top) - b ],  \label{eq:DerhGammaExplicit}
\end{align}
 where $I_d {\in \R^{d\times d}}$ is the identity matrix. Above, $\lambda(U,U)$ and its (total) derivative $\Der \lambda(U,U)$ were both defined in Lemma~\ref{lem:derLambda}. 
\end{lem}

{The next lemma asserts that}, when $\g$ is sufficiently large,~$h_\g$ is a merit (or exact penalty) function for~\eqref{eq:bm},  i.e., we can focus on minimizing the smooth function~$h_\g$, {instead of the constrained problem~\eqref{eq:bm}.} 

{While we are not aware of a precise precedent,} the proposition below is similar to~\cite[Proposition 4.22]{bertsekas2014constrained}. 

\begin{prop}[{\sc Merit function}]\label{prop:lyapunov}
Suppose that Assumption~\ref{assumption:key} holds.  Consider a matrix $\ol{U}\in \R^{d\times p}$ that is sufficiently close to $\M_b$ and satisfies $\|\ol{U}\|<\xi$, see \eqref{eq:bm} and~\eqref{eq:manifold}. For a sufficiently large  $\g$, the following statements are true:
\begin{enumerate}[label=(\roman*),leftmargin=.6cm]
    \item If $\ol{U}$ is an FOSP of $h_\g$, then $\ol{U}$ is also an FOSP of~\eqref{eq:bm} and, in particular,~$\ol{U}\in \M_b$. 
    \item If, in addition, $\ol{U}$ is an SOSP of $h_\g$, then $\ol{U}$ is also an SOSP of \eqref{eq:bm}. 
\end{enumerate}
\end{prop}

We are now {ready and fully equipped} to reach Target~\ref{o:target}.


\section{Main Result}\label{sec:mainResults}

In Sections~\ref{sec:nonConvexGeom}-\ref{sec:globalOpt}, we studied the nonconvex geometry and   optimality conditions of \eqref{eq:bm}. In Section~\ref{sec:exact}, we then introduced $h_\g$, a {(smooth)} merit function for~\eqref{eq:bm}. 

We will establish in this section that {the} gradient flow for the merit function $h_\g$, when initialized properly, converges almost surely to a global minimizer of \eqref{eq:bm} and thus achieves Target~\ref{o:target}, without getting trapped by any spurious SOSPs  present in  {the feasible set of~\eqref{eq:bm}}. 

{The main result of this work is summarized in Theorem~\ref{thm:main0} and the rest of this section is devoted to the proof of this theorem. We begin with an outline of the proof below.}

{\emph{Proof sketch of Theorem~\ref{thm:main0}.} At a high-level, we will take the following steps in the remainder of this section to ultimately prove Theorem~\ref{thm:main0}.
We will establish in this section that
\begin{enumerate}[label=(\roman*),leftmargin=.6cm]
    \item Rank does not increase along the gradient flow of~$h_\gamma$.
    \item The gradient flow does not escape from a ``capture neighborhood'' around the feasible set~$\M_b$ of~\eqref{eq:bm}.
    \item When initialized rank-deficient and within this capture neighborhood, the gradient flow of~$h_\gamma$ converges to a rank-deficient stationary point of~$h_\gamma$.
    \item Finally, because~$h_\gamma$ is a merit function for~\eqref{eq:bm},  any rank-deficient  limit point of the gradient flow is, in fact, a global minimizer of~\eqref{eq:bm}.\hfill\qedsymbol
\end{enumerate}
}

{Let us now turn to the details.} For an initialization $U_0\in \R^{d\times p}$,  the {gradient flow of~$h_\gamma$ in~\eqref{eq:hGamma} is specified as}
\begin{align}
    \dot{U}(t) =  - \nabla h_\g(U(t)),
    \qquad 
    U(0) = U_0,
    \label{eq:gradFlowH}
    \tag{gradient flow}
\end{align}
where we used the shorthand $\dot{U}(t) = \der U(t)/\der t$. 

The first lemma in this section {posits} that rank  does not increase along~\eqref{eq:gradFlowH}, as long as {the} flow remains {sufficiently close to} the feasible set of~\eqref{eq:bm}.

\begin{lem}[{\sc Rank of gradient flow}]\label{lem:rankInvariance} Suppose that Assumption~\ref{assumption:key} holds. 
For a sufficiently small $\rho_0>0$, suppose {also} that the initialization~$U_0\in \R^{d\times p}$ of~\eqref{eq:gradFlowH} satisfies~$\dist(U_0,\M_b) <\rho_0$.

Let $\tau\in (0,\infty)$ (if it exists) denote the smallest number such that $\dist(U(\tau),\M_b) = \rho_0$.
Then it holds that 
\begin{align}
\rank(U(t))\le \rank(U_0), \qquad  \text{if }t\in [0, \tau].
\label{eq:rankPer}
\end{align}
\end{lem}

{\emph{Proof sketch of Lemma~\ref{lem:rankInvariance}.}
This claim follows from writing the analytic singular value decomposition~(SVD) of $U(t)U(t)^\top$, then taking the derivative of this SVD with respect to time $t$, and finally establishing that any zero singular value of~$U(t)U(t)^\top$ remains zero afterwards.\hfill\qedsymbol }


To successfully apply Lemma~\ref{lem:rankInvariance} {for \emph{all} times} $t$, we must ensure that~\eqref{eq:gradFlowH} never escapes the $\rho_0$-neighborhood of {the feasible set}~$\M_b$ {of~\eqref{eq:bm}}. To that end, we need the following lemma, which uses the remaining freedom in 
Lemma~\ref{lem:rankInvariance} {in order} to tighten the inequality~$\dist(U(t),\M_b) \le \rho_0$ to obtain the new inequality~$\dist(U(t),\M_b) \le \rho_0/2$, for every $t\in [0,\tau]$.   

\begin{lem}[{\sc Flow remains nearby}]\label{lem:flowNearby}
   Suppose that Assumption~\ref{assumption:key} holds. Suppose also that  $\xi$ in~\eqref{eq:bm} and $\g$ in~\eqref{eq:gradFlowH} are sufficiently large, and that $\rho_0>0$ in Lemma~\ref{lem:rankInvariance} is sufficiently small. Suppose lastly that the initialization  $U_0\in \R^{d\times p}$ of~\eqref{eq:gradFlowH} is sufficiently close to~$\M_b$, see~\eqref{eq:manifold}.
   Then it holds that 
   \begin{align}
   \dist(U(t),\M_b)\le  \rho_0/2, \qquad  \text{if }t\in [0,\tau],
   \end{align}
   where $\tau$ was defined in Lemma~\ref{lem:rankInvariance}.
\end{lem}

{\emph{Proof sketch of Lemma~\ref{lem:flowNearby}.} Moving along the trajectory of~\eqref{eq:gradFlowH}  evidently reduces~$h_\gamma$. Intuitively, when~$\gamma$ is large enough, moving along this trajectory also reduces the feasibility gap~$\frac{1}{2}\|g(U)\|_2^2$, see~\eqref{eq:fgDefn},~\eqref{eq:al} and~\eqref{eq:hGamma}.

In the proof, we first quantify the above observation, i.e., formally establish that the feasibility gap does not increase along~\eqref{eq:gradFlowH} when~$\gamma$ is sufficiently large.  

The remaining technical challenge is then translating the above observation into a statement about the distance between~$U(t)$ to the manifold~$\M_b$.\hfill\qedsymbol  }

In the remainder of this section, for the sake of brevity,  we freely invoke  earlier lemmas and propositions without restating their assumptions.

Recalling the definition of $\tau$ in Lemma~\ref{lem:rankInvariance}, it immediately follows from Lemma~\ref{lem:flowNearby} that 
\begin{align}
\dist(U(t),\M_b) \le \rho_0/2,\qquad \text{if } t\ge 0.
\label{eq:flowRemainsCloseAlways}
\end{align}
That is,~\eqref{eq:gradFlowH} always remains near the feasible set~$\M_b$ of \eqref{eq:bm}, {as desired.} {The} key observation in~\eqref{eq:flowRemainsCloseAlways} will {enable} us to prove  the convergence of~\eqref{eq:gradFlowH}, after recalling the {\L}ojasiewicz’s Theorem~\cite{lojasiewicz1982trajectoires,kurdyka2000proof}. 

\begin{thm}[{\sc {\L}ojasiewicz’s Theorem}]\label{thm:lojaThm}
{If $h':\R^n \rightarrow\R$ is an analytic function and the curve 
$
[0,\infty) \rightarrow \R^n$, $t\rightarrow z(t)$
is bounded and solves the gradient flow $\dot{z}(t) = - \nabla h'(z)$, then this curve  converges to an FOSP of $h'$.}
\end{thm}

To apply Theorem~\ref{thm:lojaThm}, we proceed as follows. 
When $\rho_0>0$ is sufficiently small,  recall {from Lemma~\ref{lem:derhGamma}} that~$h_\g$ is an analytic function on the set $\{U:\dist(U,\M_b)\le \rho_0/2\}$. Let $h':\R^{d\times p}\rightarrow\R$ be the analytic continuation  of $h_\g$ from the neighborhood  $\{U:\dist(U,\M_b)\le \rho_0/2\}$ to $\R^{d\times p}$. 

Recall also from~\eqref{eq:manifold} and~\eqref{eq:flowRemainsCloseAlways}  that~\eqref{eq:gradFlowH} is bounded and solves $\dot{U}(t)= - \nabla h'(U(t))$. Therefore, by Theorem~\ref{thm:lojaThm},~\eqref{eq:gradFlowH} converges to an FOSP  of~$h'$, denoted by $\ol{U}$. To reiterate, $\ol{U}$ is both the limit point of~\eqref{eq:gradFlowH} and an
FOSP of~$h'$.

By~\eqref{eq:flowRemainsCloseAlways} {and continuity of~$\dist$, the limit point}~$\ol{U}$ {also} satisfies~$\dist(\ol{U},\M_b)\le \rho_0/2$. By construction of~$h'$ {as the analytic continuation of~$h_\gamma$},  we then {observe} that~$\ol{U}$ is also an FOSP of~$h_\g$. {To summarize, we have proved so far that}~\eqref{eq:gradFlowH} converges to an FOSP  of $h_\g$, which we have denoted by $\ol{U}$.

{Moreover,} by Theorem~3 of~\cite{panageas2016gradient}, the FOSP~$\ol{U}$ of~$h_\g$ is  almost surely also an SOSP of~$h_\g$ {\c(rather than just an FOSP)}. 
{After recalling} Proposition~\ref{prop:lyapunov} {about the relationship between~$h_\gamma$ and~\eqref{eq:bm}}, it follows that~$\ol{U}$ is also an SOSP of~\eqref{eq:bm}, provided that~$\g$ is sufficiently large.

Suppose lastly that the initialization~$U_0$  of~\eqref{eq:gradFlowH} {is rank-deficient, i.e., we have}~$\rank(U_0)<p$.
Then, using Lemma~\ref{lem:rankInvariance} {about the rank along the trajectory} and {also using} the fact that~$\{U: \rank(U)\le \rank(U_0)\}$ is a closed set, we {find} that the limit point~$\ol{U}$ of~\eqref{eq:gradFlowH} {also} satisfies 
\begin{align}
    \rank(\ol{U}) \le \rank(U_0) < p.
    \label{eq:rankInvarianceSOSP}
\end{align}
{We have established that the limit point $\ol{U}$ of \eqref{eq:gradFlowH} is almost surely a rank-deficient SOSP of~\eqref{eq:bm}. }

{Finally,} by combining Proposition~\ref{lem:noSpurious} and~\eqref{eq:rankInvarianceSOSP}, we conclude that the SOSP~$\ol{U}$ of~\eqref{eq:bm} is almost surely a global minimizer of~\eqref{eq:bm}. The  main result below summarizes our findings and achieves Target~\ref{o:target}.   

\begin{thm}[{\sc Main result}]\label{thm:main0} 
Suppose that Assumption~\ref{assumption:key} holds. Suppose also that $\xi$ in~\eqref{eq:bm} and $\g$ in~\eqref{eq:gradFlowH} are sufficiently large. 
 Suppose lastly that the initialization  $U_0\in \R^{d\times p}$ of~\eqref{eq:gradFlowH} is rank-deficient and  sufficiently close to~$\M_b$, see~\eqref{eq:manifold}.

Then~\eqref{eq:gradFlowH}   almost surely converges to a global minimizer $\ol{U}$ of \eqref{eq:bm}. Moreover,~$\ol{U}\,\ol{U}^\top$ is a global minimizer of \eqref{eq:sdp}. {Above, the notion of distance to~$\M_b$ was made precise in Definition~\ref{defn:neighb} and Remark~\ref{rem:suffClose1}.}
\end{thm}

Theorem~\ref{thm:main0} is a {theoretical (rather than practical)} recipe for successful over-parametrized matrix factorization. In particular, note that the operator~$\A$ in~\eqref{eq:measOp} is not required  {above} to satisfy the RIP, {which dominates the literature of matrix sensing~\cite{davenport2016overview}.

We also note that the existence of the initialization prescribed in Theorem~\ref{thm:main0} can be ensured under the same mild conditions that were listed after Assumption~\ref{assumption:operator}. In the remainder of this section, we justify the assumptions made in Theorem~\ref{thm:main0}.}

\begin{remark}[{\sc Assumption~\ref{assumption:key} {is minimal}}]\label{rem:licqNeeded}
 { Assumption~\ref{assumption:key} corresponds to the standard constraint qualifications for~\eqref{eq:bm}.
 More specifically,  Assumption~\ref{assumption:key} corresponds to the weakest sufficient conditions under which the KKT  conditions~\cite{nocedal2006numerical,ruszczynski2011nonlinear} are necessary for 
global optimality in~\eqref{eq:bm}, similar to~\cite[Section 5]{boumal2020deterministic} or~\cite{eftekhari2020implicit}.
 
Without Assumption~\ref{assumption:key},  
in general there cannot be be any hope of efficiently  finding a matrix~$U$ that satisfies the constraints~$\A(UU^\top)=b$ of~\eqref{eq:bm}. 

That is, without Assumption~\ref{assumption:key}, the feasibility gap~$G$ in~\eqref{eq:fgDefn} is \emph{not} necessarily dominated by its gradient $(\nabla G(U)=0 \not\Rightarrow G(U)=0)$. In {\c this scenario}, {\c a first-order optimization algorithm  cannot in general} find a feasible matrix~$U$ (a matrix $U$ that satisfies $G(U)=0$). Such peculiarities are not uncommon in nonconvex optimization~\cite{murty1985some}.

From this perspective, Assumption~\ref{assumption:key} is minimal in order to achieve Target~\ref{o:target}.

Even though Assumption~\ref{assumption:key} has several precedents within the nonconvex optimization literature~\cite{eftekhari2020implicit,boumal2016non,sahin2019inexact,bertsekas2014constrained}, we note that verifying this assumption is often difficult in practice. 
In this sense, Theorem~\ref{thm:main0} should be regarded as a theoretical result that sheds  light, for the first time, on the nonconvex geometry of over-parametrized matrix factorization.}
\end{remark}

{\begin{remark}[{\sc Initialization}]\label{rem:closeNeeded} 

{In Theorem~\ref{thm:main0}, we cannot obtain global guarantees because the feasible set of~\eqref{eq:bm} may contain spurious stationary points that can trap the gradient flow with an arbitrary initialization, see Figure~\ref{fig:vis}. 

It is then necessary to restrict the initialization in some way, e.g., the initilization near the feasible set in Theorem~\ref{thm:main0}. 

Formally, Theorem~\ref{thm:main0} is a ``capture theorem'', common in  nonconvex optimization literature~\cite{eftekhari2020implicit,boumal2016non,sahin2019inexact}, which predicates on an initialization within a specific ``capture neighborhood'' of the feasibility problem~$\min_U G(U)$, see~\eqref{eq:fgDefn}. 

In Theorem~\ref{thm:main0}, this capture neighborhood coincides with a sufficiently small neighborhood of the feasible set~$\M_b$ of~\eqref{eq:bm}, i.e., Theorem~\ref{thm:main0} applies only when~\eqref{eq:gradFlowH} is initialized near~$\M_b$. 

We do not  provide a provable scheme for finding a sufficiently feasible initialization for~\eqref{eq:gradFlowH}. In that sense, Theorem~\ref{thm:main0} should not be viewed as a practical initialization scheme for~\eqref{eq:gradFlowH} but rather a theoretical result about the nonconvex geometry of~\eqref{eq:bm}.

Nevertheless, 
as an important practical remark, we later empirically observe  that a random 
initialization~$U_0$ is often a good choice that avoids the worst-case scenario in the right panel of Figure~\ref{fig:vis}.
}

{
\begin{remark}[{\sc Local refinement results}]
Capture theorems, including Theorem~\ref{thm:main0}, are fundamentally different from  local refinement results that appear within the signal processing literature~\cite[Chapter 5]{chi2019nonconvex}. 

Indeed, note that the local refinement results rely on an initialization within a  small neighborhood of an isolated global minimizer, in which the target function is locally strongly convex. 

In contrast, even though Theorem~\eqref{thm:main0} requires a sufficiently feasible initialization,  this initialization might be \emph{far} from any global minimizer of~\eqref{eq:bm}.

Moreover, a small neighborhood of the feasible set contains \emph{all} spurious stationary points of~\eqref{eq:bm}, whereas a small neighborhood of a global minimizer will contain no other stationary points, by design.

In other words,  even though Theorem~\ref{thm:main0} requires a sufficiently feasible initialization,~\eqref{eq:gradFlowH} might have to travel far within the capture neighborhood \emph{and} avoid the spurious stationary points,
before eventually reaching a global minimizer.    
\end{remark}
}

\end{remark}}

\begin{remark}[{\sc Sufficiently small\,/\,close\,/\,large}]\label{rem:suffSmall} 
Adding to the earlier Remark~\ref{rem:suffClose1}, we note that the requirements in Theorem~\ref{thm:main0} on $\xi,\gamma,U_0$ involve certain geometric attributes of $\M_b$ in~\eqref{eq:manifold} which are difficult to estimate.

Even though the requirements on the initialization $U_0$ are specified precisely in the proofs, we chose not to  present them in the body of the paper because of their little added value and to avoid any unnecessary clutter.

\end{remark}


\section{{Discretization}}\label{sec:compAspc}


It is {\c not difficult to verify} that~\eqref{eq:gradFlowH} converges   at the rate of $1/t$. {Its limit point}  is almost surely a global minimizer of~\eqref{eq:bm}, by virtue of Theorem~\ref{thm:main0}.

The literature {\c often focuses} on  flows rather than {\c their} discretization, for the sake of simplicity and insight~\cite{gunasekar2017implicit,eftekhari2020implicit,eftekhari2020training,arora2019implicit}.
{\c Nevertheless,} {discretization of~\eqref{eq:gradFlowH} is an important computational consideration, which we now discuss {\c in this section}. 

We will not pursue an explicit Euler (or forward) discretization of~\eqref{eq:gradFlowH}. {\c We do so to avoid} stability concerns about the derivative of~$\lambda(U,U)$, see Lemma~\ref{lem:derLambda}. }



Instead of a forward discretization, {we} consider {here} a heuristic discretization of~\eqref{eq:gradFlowH}. {\c Our heuristic discretization below is} {inspired by~\cite{gao2019parallelizable} in the context of optimization  with orthogonality constraints.}

{In short,} at iteration~$k$, we  move along the direction~$-\nabla_1 L_\g(U_k,\lambda_k)$, where~$\lambda_k := \lambda(U_k,U_k)$. {\c Recall that~$\lambda(\cdot,\cdot)$ was defined in~\eqref{eq:lambdaExplicitLem}} and~$\nabla_1 L_\g$ is the partial derivative of the augmented Lagrangian {in~\eqref{eq:al}} with respect to its first argument.

The details are presented in  Algorithm~\ref{fig:alg}. 
The convergence analysis of Algorithm~\ref{fig:alg} is an important {and nontrivial} research question that lies beyond of the scope of this theoretical paper. {Nevertheless, we next present a numerical example to showcase the potential of Algorithm~\ref{fig:alg} for solving~\eqref{eq:bm}.}


\begin{algorithm}[h!]
\textbf{Input:} Symmetric $d\times d$ matrices $\{A_i\}_{i=1}^m$ and the corresponding operator $\A$ in~\eqref{eq:measOp}, vector $b\in \R^m$, integer $p$ such that $pd\ge m$, initialization $U_0\in \R^{d\times p}$, positive penalty weight $\g$  and positive step sizes $\{\eta_k\}_k$.

\vspace{10pt}
\noindent Set $k=0$. Until convergence, repeat
\begin{enumerate}

\item Update the dual variables as $\lambda_k := \lambda(U_k,U_k)$, see~\eqref{eq:lambdaExplicitLem}.

\item Update the primal variables as 
\begin{align*}
U_{k+1} & = (1-\eta_k) U_k \nonumber\\
& \,\, + \sum_{i=1}^m \eta_k \left( \lambda_{k,i} - \frac{\g}{2} (\langle A_i, U_k U_k^\top \rangle - b_i) \right) A_i U_k, 
\end{align*} 
where $\lambda_{k,i}$ and $b_i$ are the $i^{\text{th}}$ entries of the vectors $\lambda_k$ and $b$, respectively. 
\item $k\leftarrow k+1$

\end{enumerate}

\caption{
{Discretization of~\eqref{eq:gradFlowH}}
\label{fig:alg}}
\end{algorithm}

\section{Numerical Example}\label{sec:numerics}

This section presents a {small} numerical example 
 {\c that shows} the potential of Algorithm~\ref{fig:alg} for over-parametrized matrix factorization. A comprehensive numerical study remains as  a future research target, alongside developing a convergence theory for Algorithm~\ref{fig:alg}. 

{Recall the setup of~\eqref{eq:bm}.
In our {\c numerical} example, we set~$d=15$,~$m=30$,~$p=\lceil \sqrt{2m}\rceil$,~$\xi\gg 1$. {invoke} the Pataki's lemma~\cite{pataki1998rank}\cite[Theorem 6.1]{polik2007survey} {\c to verify} that Assumption~\ref{assumption:operator} is fulfilled. (In particular, we are {\c indeed} in the over-parametrized regime, see Assumption~\ref{assumption:operator}.\ref{assumption:badRegime}.) 

{\c We also} choose the linear operator~$\A:\R^{d\times d}\rightarrow\R^m$ and the vector~$b\in \R^m$ both randomly. {\c More specifically, the upper triangle entries of every matrix $A_i\in \R^{d\times d}$ are independently drawn from the zero-mean and unit-variance Gaussian distribution. Similarly, $b$ is a standard Gaussian random vector. Moreover,~$\{A_i\}_{i=1}^m$ and~$b$ are independent from one another.}

{\c Because Assumption~\ref{assumption:operator}.\ref{assumption:infoLevel} is fulfilled,~\eqref{eq:sdp} is a tight relaxation of~\eqref{eq:bm}. In particular,} the optimal value of~\eqref{eq:bm} coincides with {\c the optimal value} of the convex problem~\eqref{eq:sdp}. {\c As a benchmark,} we can use CVX~\cite{cvx,gb08} to solve~\eqref{eq:sdp} {\c and obtain the common optimal value of~\eqref{eq:sdp} and~\eqref{eq:bm}.}}




We {then} attempt to solve~\eqref{eq:bm} with  Algorithm~\ref{fig:alg}, where {\c we set}  the  {penalty weight and step sizes to~$\g = 100$ and~$\eta_k =  2\cdot 10^{-5}$ for every $k$. We use three different initializations for Algorithm~\ref{fig:alg}, detailed below:
\begin{enumerate}[label=(\roman*),leftmargin=.6cm]
\item A deterministic initialization, where~$U_0\in\R^{d\times p}$ is filled by zeros and ones.

\item A ``partial oracle'' initialization, where~$U_0$ contains only one correct column of a global minimizer {\c of~\eqref{eq:bm}}. The remaining entries of~$U_0$ are all set to one. {\c Here, we can obtain a global minimizer of~\eqref{eq:bm} by taking the square root of CVX's output of~\eqref{eq:sdp}.}

\item {\c $U_0$ is a standard Gaussian random matrix.}

\end{enumerate}
In all cases, we then normalize~$U_0$ to ensure that~$\|U_0\|_{\F}=3$. This last step is {\c for} convenience and allows us to use the same step size for all three initializations. }  

Figure~\ref{fig:1} shows the feasibility gap and the target value of~\eqref{eq:bm} {across the iterations of Algorithm~\ref{fig:alg}, using the above three  initializations. For comparison, the optimal value of problem~\eqref{eq:bm} is shown with a dashed line. The MATLAB code will be made available with the paper.

Algorithm~\ref{fig:alg} with both the deterministic and partial oracle initializations converges to stationary points of~$h_\gamma$ but neither of these two limit points is  feasible for~\eqref{eq:bm}. 

{\c That is, Algorithm~\ref{fig:alg} fails for both initializations to produce an output that satisfies the constraints in~\eqref{eq:bm}.} {\c In both cases, note  that} the output of Algorithm~\ref{fig:alg} is not a  spurious stationary point of~\eqref{eq:bm}. 

The failure of Algorithm~\ref{fig:alg} with these two initializations hints at the complex landscape of the merit function~$h_\gamma$ and, in turn, the difficulty of solving~\eqref{eq:bm} in the over-parametrized regime. 
}

{{\c Nevertheless}, Algorithm~\ref{fig:alg} with a generic initialization successfully solves~\eqref{eq:bm} to global optimality in Figure~\ref{fig:1}. {\c Remarkably, our experience was} that a generic initialization always avoids the worst-case scenarios, such as the right panel of Figure~\ref{fig:vis} or the first two initializations in Figure~\ref{fig:1}. This observation is briefly discussed in the next section.}

It is also worth noting that we found it helpful in our simulations to stabilize Algorithm~\ref{fig:alg} by replacing $K(U_k)$ in~\eqref{eq:lambdaExplicitLem} by $K(U_k)+10^{-9}I_m$, where $I_m$ is the identity matrix. 

Lastly, note that it is difficult to verify the manifold requirement for~$\M_b$ or to numerically identify its capture neighborhood.
In this sense, Theorem~\ref{thm:main0} should be regarded as a theoretical contribution that sheds light, for the first time, on the nonconvex geometry of~\eqref{eq:bm} in the over-parametrized regime, rather than a practical initialization scheme for Algorithm~\ref{fig:alg}. In practice, a generic initialization seems to be an excellent choice for Algorithm~\ref{fig:alg}.

\begin{figure}
\begin{center}
    \includegraphics[width=8cm,height=6cm]{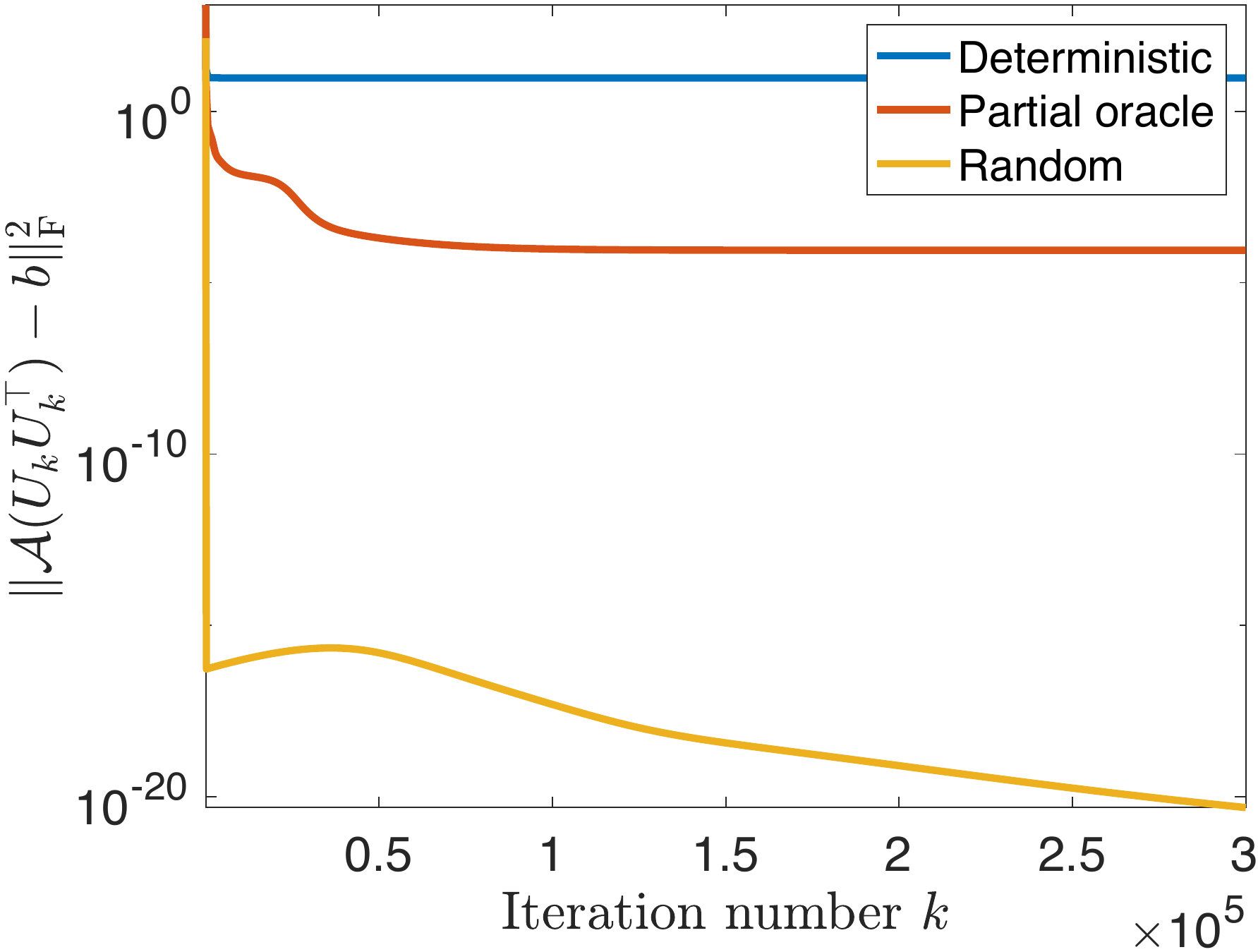}
\vspace{10pt}    

    \includegraphics[width=8cm,height=6cm]{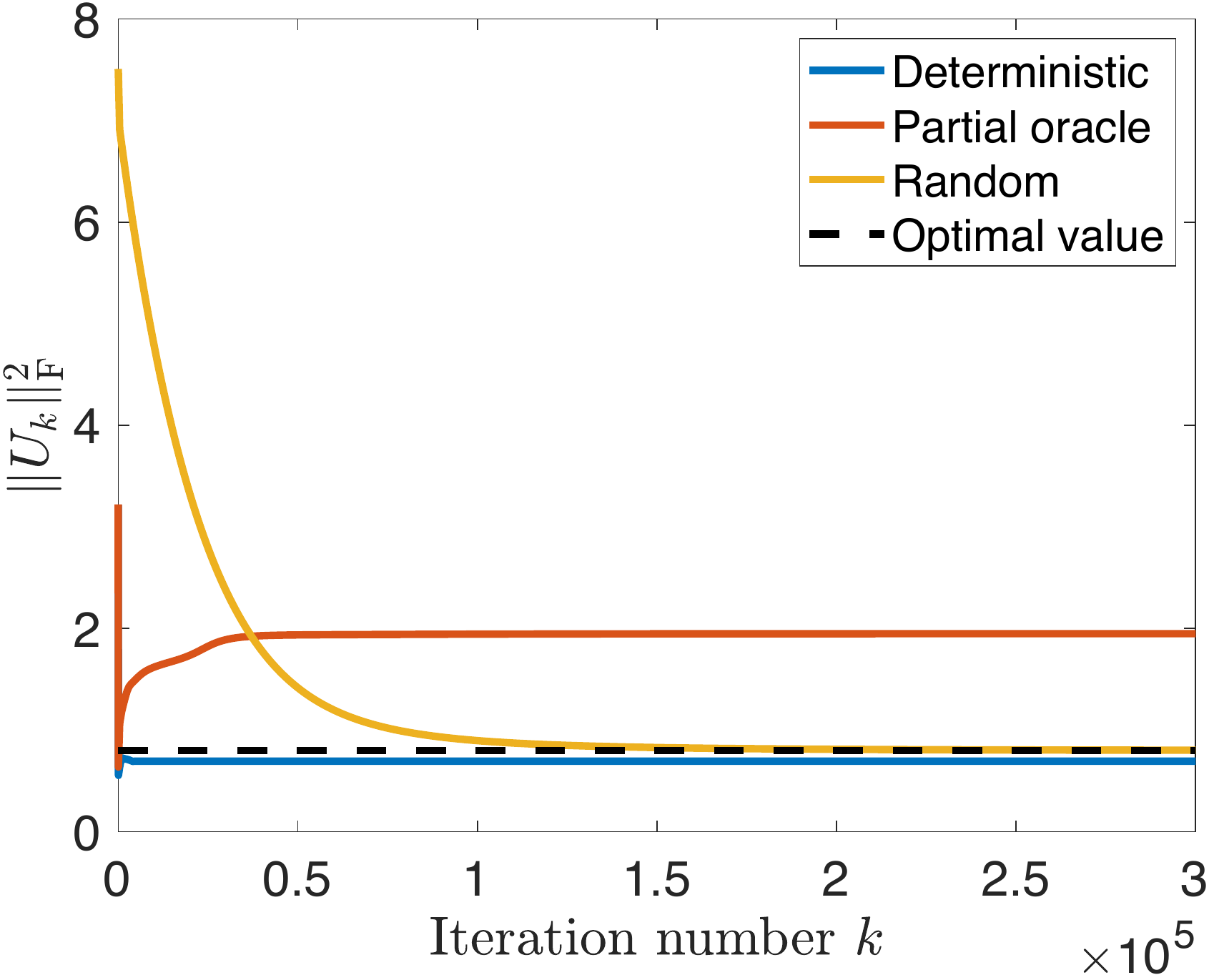}
    \caption{{For the numerical example of Section~\ref{sec:numerics}, this figure shows the feasibility gap (top panel) and the target value (bottom panel) of~\eqref{eq:bm} across the iterations of Algorithm~\ref{fig:alg}, using three different initializations.}} 
    \label{fig:1}
\end{center}
\end{figure}

    

\section{{Final Thoughts}}

This work  raises a few intriguing questions. First, as mentioned earlier, a stable and provable discretization of~\eqref{eq:gradFlowH} is an interesting and nontrivial future research question. {In particular, the convergence analysis of Algorithm~\ref{fig:alg} remains an open problem.}

{Second, in our numerical examples, recall that  we came across both spurious stationary points of~\eqref{eq:bm} and infeasible stationary points of its merit function~$h_\gamma$, see the right panel of Figures~\ref{fig:vis} and~\ref{fig:1}, respectively. {\c These observations suggest that~\eqref{eq:bm} is a difficult problem to solve in the over-parametrized regime.}

{\c At the same time,} what explains the surprising success of Algorithm~\ref{fig:alg} {\c when the linear operator $\A$ and the
initialization are both generic}? Answering this interesting question might require new technical tools beyond our toolbox {\c in this paper}.} 




\appendices
\section{Proof of Lemma~\ref{lem:flowNearby}}

An important ingredient of the proof is the following technical lemma,
which states that the feasibility gap is non-increasing along~\eqref{eq:gradFlowH}. 
\begin{lem}[{\sc Flow remains nearly feasible}]\label{lem:feasAlongFlow}
Suppose that the same assumptions made in Lemma~\ref{lem:flowNearby} are fulfilled. Then it holds that 
\begin{align}
    \|g(U_t)\|_2 \le \|g(U_0)\|_2, \qquad \text{if } t\in[0, \tau],
    \label{eq:uppBndGLem}
\end{align}
where $g$ was defined in~\eqref{eq:fgDefn}. 
\end{lem}
Before proving Lemma~\ref{lem:feasAlongFlow} in a later appendix, let us here complete the proof of Lemma~\ref{lem:flowNearby}. The technical challenge ahead is translating the upper bound in~\eqref{eq:uppBndGLem} into an upper bound on the distance from $U_t$ to the set $\M_b$ in~\eqref{eq:manifold}. 
We first write that 
\begin{align}
 \|g(U_0)\|_2 & \ge \|g(U_t)\|_2   \qquad \text{(see Lemma \ref{lem:feasAlongFlow})} \nonumber\\
& = \frac{1}{2}\|\A(U_t U_t^\top) - b\|_2 
  \qquad \text{(see \eqref{eq:fgDefn})},
  \label{eq:beforeRIPFlow}
\end{align}
for every $t\in[0, \tau]$. To lower bound the last norm above, the idea is to replace $b$ above with the image of a particular point in $\M_b$ under the map $U\rightarrow\A(UU^\top)$, see~\eqref{eq:manifold}. That particular point is  the projection of $U_t$ onto $\M_b$, as detailed next:
When $\xi$ is sufficiently large, let $V_t\in \M_b$ denote a projection of $U_t$ onto $\M_b$ in~\eqref{eq:manifold}. That is, when $\xi$ is sufficiently large, there exists $V_t\in \M_b$ such that 
\begin{align}
     \| U_t - V_t \|_\F & = \dist(U_t, \M_b)  \nonumber\\
     & \le  \|U_t - V\|_\F ,\qquad \text{if } V\in \M_b, 
    \label{eq:projUt}
\end{align}
where the second inequality above follows from~\eqref{eq:distDefn}. 
Recall also that the normal space of the smooth manifold~$\M_b$ was specified in~\eqref{eq:normalSpace}. In particular, note that $U_t - V_t\in \N_{V_t}\M_b$ by~\eqref{eq:projUt}. Equivalently, there exists a vector $\a_t\in \R^m$ such that 
\begin{align}
    U_t 
    & = V_t + (\Der g(V_t))^*[\a_t]. 
    \qquad \text{(see \eqref{eq:normalSpace}, \eqref{eq:projUt})}
    \label{eq:diffNormal}
\end{align}
It follows that 
\begin{align}
    \dist(U_t,\M_b) & =  \|U_t - V_t\|_\F   \qquad \text{(see \eqref{eq:projUt})} \nonumber\\
    & = \| (\Der g(V_t))^*[\a_t]\|_\F  \qquad \text{(see \eqref{eq:diffNormal})} \nonumber\\
    & \le \| (\Der g(V_t))^*\| \cdot \|\a_t\|_2 \nonumber\\
    & =  \|\Der g(V_t)\| \cdot \|\a_t\|_2   \nonumber\\
    & \le \xi  \|\A\| \cdot \|\a_t\|_2 .
    \label{eq:distToAlpha}
\end{align}
The last line above follows from the observation that 
\begin{align}
    \|\Der g(V_t)\|   
    & \le \|V_t\| \|\A\| 
    \qquad \text{(see \eqref{eq:defnC})}\nonumber\\
    & \le \xi  \|\A\|.
    \qquad (\text{\eqref{eq:manifold} and }V_t\in \M_b)
    \label{eq:bndGVTNorm}
\end{align}
Under  Assumption~\ref{assumption:key},
we can also  write a converse for~\eqref{eq:distToAlpha}. That is,   
\begin{align}
    \dist(U_t, \M_b) & = \|U_t -V_t\|_\F 
    \qquad \text{(see \eqref{eq:projUt})} \nonumber\\
    & = \| (\Der g(V_t))^*[\a_t]\|_2 
    \qquad \text{(see \eqref{eq:diffNormal})}
    \nonumber\\
    & \ge \s_m(\Der g(V_t))\cdot  \|\a_t\|_2  
    \qquad (m \le pd)
    \nonumber\\    
    & \ge \s_{m}(\M_b) \cdot \| \a_t\|_2 .
    \qquad (V_t\in \M_b) 
    \label{eq:converseToDistToAlpha}
\end{align}
For brevity, above we set 
\begin{align}
\s_m(\M_b) := \min\{\s_m(\Der g(U)): U\in \cl(\M_b)\} >0. 
\label{eq:sminMb}
\end{align}
Note that $\s_{m}(\M_b)$ above is positive under Assumption~\ref{assumption:key} because  $\M_b$ in~\eqref{eq:manifold} is bounded. The two inequalities in~\eqref{eq:distToAlpha} and~\eqref{eq:converseToDistToAlpha} relate $\dist(U_t,\M_b)$ to $\a_t$. These two inequalities will be useful for us later in the proof.

To continue, note also that we have $\A(V_tV_t^\top)=b$ because $V_t\in \M_b$ by construction,  see also~\eqref{eq:manifold}. Using this last observation, we now revisit~\eqref{eq:beforeRIPFlow} and write that 
\begin{align}
    & 2\| g(U_0)\|_2  \nonumber\\
    & \ge \| \A(U_t U_t^\top) - b\|_2 \qquad \text{(see \eqref{eq:beforeRIPFlow})}\nonumber\\
    & = \|\A( U_t U_t^\top - V_t V_t^\top ) \|  
    \qquad \text{(see \eqref{eq:manifold})}\nonumber\\
    & = \Big\| 2\A ( (\Der g(V_t))^*[\a_t] V_t^\top    ) \nonumber\\
    & \qquad + \A ( (\Der g(V_t))^*[\a_t] \cdot ((\Der g(V_t))^*[\a_t])^\top )\Big\|_2,
    \label{eq:beforeDist0}
\end{align}
where we used \eqref{eq:diffNormal} in the last identity above. There, we also benefited from the symmetry of~$\{A_i\}_{i=1}^m$ in~\eqref{eq:measOp}. By applying the reverse triangle inequality to the last line above, it follows from \eqref{eq:beforeDist0} that 
\begin{align}
     2 \|g(U_0)\|_2 
    & \ge 2 \| \A ( (\Der g(V_t))^*[\a_t] V_t^\top    ) \|_2  \nonumber\\
    & - \| \A ( (\Der g(V_t))^*[\a_t] \cdot ((\Der g(V_t))^*[\a_t])^\top )\|_2
    \nonumber\\
    & \ge  2 \| \A ( (\Der g(V_t))^*[\a_t] V_t^\top    ) \|_2 \nonumber\\
    & - 
    \|\A\| \|\cdot  \Der g(V_t)\|^2\cdot  \|\a_t\|^2_2  \nonumber\\
    & =  2\| ((\Der g(V_t))\circ (\Der g(V_t))^*)[\a_t]\|_2 \nonumber\\
    & - 
    \|\A\| \cdot \| \Der g(V_t)\|^2\cdot  \|\a_t\|^2_2, \qquad \text{(see \eqref{eq:defnC})}
    \label{eq:beforeDist1}
\end{align}  
where~$\circ$ denotes the composition of two operators. Recalling the fact that $V_t\in \M_b$ by construction, the last line above can be lower bounded  as
\begin{align}    
    2\|g(U_0)\|_2
    & \ge 2\s_m(\M_b)^2 \|\a_t\|_2\nonumber\\
    & \quad - \|\A\| \cdot \xi^2 \|\A\|^2 \cdot \|\a_t\|^2_2
    \quad \text{(see \eqref{eq:bndGVTNorm}, \eqref{eq:sminMb})} \nonumber\\
    & \ge \s_m(\M_b)^2 \|\a_t\|_2,
    \label{eq:beforeDist}
\end{align}
where the last line above holds if $\|\a\|_2$ is sufficiently small, i.e., the last line above holds if 
\begin{align}
    \|\a_t\|_2 \le \frac{\s_m(\M_b)^2}{ \xi^2\|\A\|^3}.
    \label{eq:beforeDist2}
\end{align}
In view of~\eqref{eq:distToAlpha} and~\eqref{eq:converseToDistToAlpha}, it follows from~\eqref{eq:beforeDist} and~\eqref{eq:beforeDist2} that 
\begin{align}
   & \dist(U_t, \M_b) \le \frac{\s_m(\M_b)^3}{\xi^2 \|\A\|^3}   \nonumber\\
   & \Longrightarrow  \dist(U_t, \M_b) \le \frac{2\xi \|\A\| \|g(U_0)\|_2}{\s_m(\M_b)^2}.
   \label{eq:relationLargeSmall}
\end{align}
Recalling the assumptions of Lemma~\ref{lem:flowNearby}, note that $\dist(U_t,\M_b) \le \rho_0$ for every~$t\in [0,\tau]$. Suppose that $\rho_0$ is sufficiently small, i.e., take 
\begin{align}
    \rho_0 \le  \frac{\s_m(\M_b)^3}{\xi^2 \|\A\|^3}.
\end{align}
Then,~\eqref{eq:relationLargeSmall} immediately implies that   
\begin{align}
    \dist(U_t, \M_b) \le \frac{2\xi \|\A\| \cdot \|g(U_0)\|_2}{\s_m(\M_b)^2},
    \qquad \text{if } t\in[0, \tau].
    \label{eq:noUppBndonG0}
\end{align}
Let us now rephrase the right-hand side of~\eqref{eq:noUppBndonG0}. Specifically, we next  upper bound $\|g(U_0)\|_2$ above by the initial distance to the manifold, i.e.,~$\dist(U_0,\M_b)$. Recall from~\eqref{eq:projUt} that $V_{0}\in \M_b$ denotes a projection of $U_0$ on $\M_b$ in~\eqref{eq:manifold}.
In particular, $V_0\in \M_b$ implies that $\A(V_0 V_0^\top) = b$ by~\eqref{eq:manifold}.
Using this last observation and~\eqref{eq:projUt}, we   bound~$\|g(U_0)\|_2$ as 
\begin{align}
    \|g(U_0)\|_2 & = \frac{1}{2} \| \A(U_0U_0^\top) - b\|_2 
    \qquad \text{(see \eqref{eq:fgDefn})}
    \nonumber\\
    & = \frac{1}{2} \| \A(U_0U_0^\top- V_0 V_0^\top ) \|_2
    \qquad \text{(see \eqref{eq:manifold})} 
    \nonumber\\
    & \le \frac{1}{2} \|\A\|\cdot  \|  U_0U_0^\top - V_0 V_0^\top\|_\F 
     \nonumber\\
    & \le \frac{1}{2}\|\A\|\cdot  (\|U_0\|+\|V_0\|) \cdot \|U_0 - V_0\|_\F \nonumber\\
    & \le \|\A\|\xi \cdot  \|U_0 - V_0\|_\F \nonumber\\
    & = \|\A\|\xi \cdot  \dist(U_0,\M_b),
    \qquad \text{(see \eqref{eq:projUt})}
    \label{eq:lowBndGU0}
\end{align}
where the second-to-last line above assumes that $\xi$ is sufficiently large, i.e.,~$\xi \ge \|U_0\|$.
The second-to-last line in~\eqref{eq:lowBndGU0} also uses the fact that  $V_0\in \M_b$ in~\eqref{eq:manifold} and, in particular,~$\|V_0\|\le \xi$. By combining~\eqref{eq:noUppBndonG0} and~\eqref{eq:lowBndGU0}, we arrive at 
\begin{align*}
    \dist(U_t, \M_b) \le \frac{2\xi^2 \|\A\|^2 \dist(U_0,\M_b) }{\s_m(\M_b)^2}, \qquad  \text{if } t \in [0,\tau],
\end{align*}
provided that $\xi$ is sufficiently large.
By setting $\dist(U_0,\M_b)$ sufficiently small, i.e., by taking
\begin{align}
    \dist(U_0,\M_b) \le \frac{ \s_m(\M_b)^2\rho_0 }{4\xi^2 \|\A\|^2},
\end{align}
we can ensure that 
\begin{align}
    \dist(U_t, \M_b) \le \rho_0/2, \qquad \text{if } t \in [0,\tau].
\end{align}
Above, $\rho_0$ was defined in Lemma~\ref{lem:rankInvariance}. This completes the proof  of  Lemma~\ref{lem:flowNearby}. 

\section{Proof of Lemma~\ref{lem:feasAlongFlow}}

Recall that $\rho_0$  denotes the radius of the neighborhood in Lemma~\ref{lem:rankInvariance}. 
Recall also from~(\ref{eq:fgDefn}),~\eqref{eq:adjO}  and~(\ref{eq:DerhGammaExplicit}) that
\begin{align}
    \nabla h_\g(U_t) & = (I-\A^*(\lambda(U_t,U_t))) U_t + \g (\Der g(U_t))^* [g(U_t)] \nonumber\\
    & - (\Der \lambda(U_t,U_t))^*[g(U_t)],
    \label{eq:DerhGammaExplicitRecall2}
\end{align}
for every $t\in [0, \tau]$. 
Recall from~\eqref{eq:fgDefn} the definition of the (scaled) feasibility gap $G:\R^{d\times p}\rightarrow\R$. 
To study the evolution of the feasibility gap along~\eqref{eq:gradFlowH}, we write that 
\begin{align}
    & \frac{\der G(U_t)}{\der t} \nonumber\\
    & = \langle \nabla G(U_t), \dot{U}_t  \rangle 
    \qquad \text{(chain rule)}
    \nonumber\\
    & = -  \langle \nabla G(U_t), \nabla h_\g(U_t)  \rangle
    \qquad \text{(see \eqref{eq:gradFlowH})} \nonumber\\
    & = - \langle (\Der g(U_t))^*[g(U_t)], \nabla h_\g(U_t) \rangle \qquad \text{(\eqref{eq:fgDefn} and chain rule)} \nonumber\\
    & = - \langle (\Der g(U_t))^*[g(U_t)], (I_d - \A^*(\lambda(U_t,U_t)) ) U_t \rangle \nonumber\\
    & - \g \| (\Der g(U_t))^* [g(U_t)]\|_{\F}^2 \nonumber\\
    & + \langle (\Der g(U_t))^*[g(U_t)], (\Der \lambda(U_t,U_t))^*[g(U_t)] \rangle 
    \qquad \text{(see \eqref{eq:DerhGammaExplicitRecall2})} \nonumber\\
    & \le  \|\Der g(U_t)\|\cdot  \|g(U_t)\|_2 \cdot \|I_d - \A^*(\lambda(U_t,U_t)) \| \cdot \|U_t\|_\F  \nonumber\\
    & - \g (\s_m (\Der g(U_t)))^2 \cdot  \|g(U_t)\|_{2}^2 \nonumber\\
    & + \| \Der g(U_t)\|  \cdot  \|\Der \lambda(U_t,U_t)\| \cdot  \|g(U_t)\|_2^2 \qquad \text{(Cauchy-Schwarz)} \nonumber\\
    & = \sqrt{2G(U_t)} \cdot \Big(  \|\Der g(U_t)\| \cdot \|I_d - \A^*(\lambda(U_t,U_t)) \| \cdot \|U_t\|_\F \nonumber\\
    & -  \g (\s_m (\Der g(U_t)))^2 \sqrt{2 G(U_t)} \nonumber\\
    &   +  \|\Der g(U_t)\|  \cdot  \|\Der \lambda(U_t,U_t)\|  \sqrt{2 G(U_t)} \Big).
    \quad \text{(see \eqref{eq:fgDefn})}
    \label{eq:evolveFeasGap}
\end{align}
To control the terms in the last identity above, we make two observations:
\begin{enumerate}[leftmargin=*]
    \item Recall that both $\cl(\M_b)$ in~\eqref{eq:manifold} and its neighborhood $\{U: \dist(U,\M_b)\le \rho_0\}$  are compact sets. Because $\rho_0<\rho$ by design, note also that $g$ in~\eqref{eq:fgDefn} and $\lambda$ are both continuously-differentiable functions on $\{U: \dist(U,\M_b)\le \rho_0\}$, see Lemma~\ref{lem:derLambda}. Here, $\rho$ is the radius of the neighborhoods both in Assumption~\ref{assumption:key} and Lemma~\ref{lem:derLambda}.

    It follows that the functions $U\rightarrow\|\Der g(U)\|$, $U\rightarrow \|I_d - \A^*(\lambda(U,U))\|$, $U\rightarrow\|U\|_\F$, and $U\rightarrow \|\Der \lambda(U,U)\|$ are all bounded on the set $\{U: \dist(U,\M_b)\le \rho_0\}$. Consequently,  the corresponding terms in the last identity of~\eqref{eq:evolveFeasGap} are bounded on the interval~$[0,\tau]$. 

\item Moreover,  Assumption~\ref{assumption:key} and $\rho_0<\rho$ together imply that the function $U\rightarrow\s_m(\Der g(U))$ is bounded away from zero on the set $\{U:\dist(U,\M_b)\le \rho_0\}$. Consequently, the corresponding term in the last identity of~\eqref{eq:evolveFeasGap} is bounded away from zero on the interval~$[0,\tau]$. 

\end{enumerate}
If $G(U_{t_0})=0$ for $t_0\in [0,\tau]$, we find from~\eqref{eq:evolveFeasGap} that $G(U_t)=0$ for every $t\in [t_0,\tau]$. That is, once feasible,~\eqref{eq:gradFlowH} remains feasible afterwards. On the other hand, in view of the above two observations, we see  that  $\der G(U_t)/\der t < 0$ for every $t\in[0, t_0)$, provided that  $\g$ is sufficiently large. We conclude that~$\der G(U_t)/\der t<0$ for every $t\in [0,\tau]$, provided that~$\gamma$ is sufficiently large.
This completes the proof of Lemma~\ref{lem:feasAlongFlow}.

\bibliographystyle{unsrt}
\bibliography{References}

\newpage
\appendices

\title{Supplementary Material} 
\author{}
\date{}
\emptythanks
\maketitle

\section{Proof of Lemma \ref{lem:derLambda}}

In order to find an explicit expression for $\lambda(U,U)$ in~\eqref{eq:lambdaDefn}, we first calculate $((\Der g(U))^*)^\dagger$ as follows. Under Assumption~\ref{assumption:operator}.\ref{assumption:badRegime} and after recalling~\eqref{eq:adjO}, note that $((\Der g(U))^*)^\dagger :\R^{d\times p}\rightarrow\R^m$ is specified as
\begin{align}
    & ((\Der g(U))^*)^\dagger [\D] := (K(U))^{\dagger} \A(\D U^\top), 
     \nonumber\\
    & K(U) := [ \langle A_i U, A_j U \rangle ]_{i,j=1}^m \in \R^{m\times m}, \label{eq:defnKernel}
\end{align}
for $\D\in \R^{d\times p}$. 
We then recall the definition of $\lambda(U,U)$ in~\eqref{eq:lambdaDefn} to write that
\begin{align}
    \lambda(U,U) & =  ((\Der g(U))^*)^\dagger [U] \qquad \text{(see \eqref{eq:lambdaDefn})} \nonumber\\
    & = (K(U))^{\dagger} \A(U U^\top),\qquad \text{(see \eqref{eq:defnKernel})}
    \label{eq:lambdaExplicit}
\end{align}
which proves~\eqref{eq:lambdaExplicitLem}. 

We next prove the second claim in Lemma~\ref{lem:derLambda}. 
Recall that Assumption~\ref{assumption:key} holds and let $\rho$ denote the radius of the neighborhood specified in Assumption~\ref{assumption:key}. 
That is,
\begin{align}
    \rank(\Der g(U)) = m, \qquad \text{if } \dist(U,\M_b)
    \le \rho.
    \label{eq:rankDerGFix}
\end{align}
Equivalently, by definition of $(\Der g(U))^*$ in~\eqref{eq:adjO}, the matrices $\{A_i U\}_{i=1}^m$ are linearly independent for every $U$ such that $\dist(U,\M_b)\le  \rho$. From~\eqref{eq:defnKernel} and~\eqref{eq:rankDerGFix}, it immediately follows that 
\begin{align}
    \rank(K(U)) = m, \qquad \text{if }\dist(U,\M_b) \le \rho. \label{eq:rankKConstant}
\end{align}
Note that $K(U)$  in~\eqref{eq:lambdaExplicit} is an analytic function of $U$ in $\R^{d\times p}$. 
By~\eqref{eq:rankKConstant} and the boundedness of $\M_b$ in~\eqref{eq:manifold}, $(K(U))^{-1}$ is also an analytic function of $U$ on the set 
\begin{align}
  \M_{b,\rho}:=\{U:\dist(U,\M_b)<  \rho\}. \label{eq:neighApp}
\end{align}
In view of~\eqref{eq:lambdaExplicit}, $\lambda(U,U)$ is also an analytic function of $U$ on the set $\M_{b,\rho}$, which proves the second claim in Lemma~\ref{lem:derLambda}.

We next prove the third and final claim in Lemma~\ref{lem:derLambda}. 
To compute the derivative of $\lambda(U,U)$ with respect to $U$, we begin by computing the (total) derivatives of $K(U)$ and $(K(U))^{-1}$, see~\eqref{eq:defnKernel}. For $\D\in \R^{d\times p}$, note that the directional derivative of $K(U)$ at $U$ and along the direction $\D$ is given by
\begin{align}
    \Der K(U)[\D] = 2 \l[\l\langle A_i \D, A_j U\r\rangle  \r]_{i,j=1}^m =: 2  \wt{K}(U,\D). \label{eq:derKExplicit}
\end{align}


To compute the directional derivative of $(K(U))^{-1}$ at $U\in \M_{b,\rho}$ and along $\D\in \R^{n\times p}$,  we compute the directional derivative of both sides of the identity $K(U) \cdot (K(U))^{-1} = I_m$, along a direction $\D\in \R^{d\times p}$. That is, 
\begin{align*}
\Der K(U)[\D] \cdot (K(U))^{-1} + K(U) \cdot \Der (K(U))^{-1}[\D] = 0,
\end{align*}
which, after rearranging, yields that 
\begin{align}
    & \Der (K(U))^{-1}[\D] \nonumber\\
    & = - (K(U))^{-1} \cdot \Der K(U)[\D] \cdot (K(U))^{-1} \nonumber\\
    & = -2 (K(U))^{-1} \cdot \wt{K}(U,\D) \cdot (K(U))^{-1}.
    \quad \text{(see \eqref{eq:derKExplicit})}
            \label{eq:derKdaggerFinal}
\end{align}
Having computed in~\eqref{eq:derKdaggerFinal} the directional derivative of $ (K(U))^{-1}$ at $U\in \M_{b,\rho}$, we are now ready to compute the derivative of $\lambda(U,U)$ with respect to $U$ as follows.  
Using the definition of $\lambda(U,U)$ in~\eqref{eq:lambdaExplicit} and  for a direction $\D\in \R^{d\times p}$, the directional derivative of $\lambda(U,U)$ along $\D$ is given by 
\begin{align}
    & \Der \lambda(U,U)[\D] \label{eq:lambdaDerFinalLemma}\\ 
    & = \Der (K(U))^{-1} [\D] \cdot \A(UU^\top) + 2(K(U))^{-1} \A(\D U^\top) \nonumber\\
    & = -2 (K(U))^{-1} \wt{K}(U,\D) (K(U))^{-1} \A(UU^\top) 
    \nonumber\\
    & \qquad + 2 (K(U))^{-1} \A(\D U^\top) 
    \qquad \text{(see \eqref{eq:derKdaggerFinal})}
    \nonumber\\
    & = 2 (K(U))^{-1} \l( -\wt{K}(U,\D) (K(U))^{-1} \A(UU^\top) 
    +  \A(\D U^\top)   \r). \nonumber
\end{align}
This completes the proof of Lemma~\ref{lem:derLambda}.

\section{Proof of Proposition~\ref{prop:lyapunovPre}}
Recalling~(\ref{eq:manifold}), let us fix $U\in \R^{d\times p}$ such that $\dist(U,\M_b) \le  \rho$ and $\|U\|<\xi$, where $\rho$ is the radius of the neighborhood in Assumption~\ref{assumption:key}. Recall also~\eqref{eq:al}. For $\lambda'\in \R^m$, note that the gradient of the augmented Lagrangian with respect to its first argument is specified as 
\begin{align}
    & \nabla_1 L_\g (U,\lambda') \nonumber\\
    & = \nabla f(U) - (\Der g(U))^*[\lambda']  + \g (\Der g(U) )^*[g(U)].
    \label{eq:gradOne}
\end{align}
With the choice of $\lambda' = \lambda(U,U)$ from~\eqref{eq:lambdaDefn}, we rewrite~\eqref{eq:gradOne} as
\begin{align}
   & \nabla_1  L_\g (U,\lambda(U,U))  \nonumber\\
    & = \nabla f(U) - (\Der g(U))^*[\lambda(U,U)] + \g (\Der g(U) )^*[g(U)] 
    \nonumber\\
    & = \nabla f(U) - \l((\Der g(U))^*\circ ((\Der g(U))^*)^\dagger\r) [U]  \nonumber\\
    & \qquad + \g (\Der g(U))^*[g(U)]
    \qquad \text{(see \eqref{eq:lambdaDefn})} \nonumber\\
    & = \nabla f(U) - \l((\Der g(U))^*\circ ((\Der g(U))^*)^\dagger\r) [\nabla f(U)] \nonumber\\
    &\qquad  + \g (\Der g(U))^*[g(U)]
    \qquad \text{(see \eqref{eq:fgDefn})} \nonumber\\
    & = (\mathrm{Id}- (\Der g(U))^*\circ ((\Der g(U))^*)^\dagger ) [\nabla f(U)] \nonumber\\
    & \qquad + \g (\Der g(U))^*[g(U)],
    \label{eq:gradALOne}
\end{align}
where $\mathrm{Id}$ is the shorthand for the identity map. 
Note that the two terms in the last line above are in fact orthogonal to one another; one is in the range of the operator $(\Der g(U))^*$  and the other is orthogonal to  $\range((\Der g(U))^*)$. In particular, $\nabla_1 L_\g(U,\lambda(U,U))=0$ implies that both 
\begin{align}
    & (\mathrm{Id}- (\Der g(U))^*\circ ((\Der g(U))^*)^\dagger ) [\nabla f(U)] = 0,\nonumber\\
    & \text{and } (\Der g(U))^*[g(U)] = 0.
    \label{eq:nullSpaceNotRemoved}
\end{align}
Moreover, recall the earlier assumption that $\dist(U,\M_b)\le  \rho$, where  $\rho$ is the radius of the neighborhood of $\M_b$ in Assumption~\ref{assumption:key}. From this assumption, it follows that $\Der g(U):\R^{d\times p}\rightarrow\R^m$ is a rank-$m$ linear operator.   In particular, the operator $\Der g(U)^*$ has a trivial null-space.  This  observation allows us to simplify the second identity in~\eqref{eq:nullSpaceNotRemoved}. More specifically, we find that both 
\begin{align}
    & (\mathrm{Id}- (\Der g(U))^*\circ ((\Der g(U))^*)^\dagger ) [\nabla f(U)] = 0 \nonumber\\
    &\text{and }  g(U) = 0.
    \label{eq:preGrad}
\end{align}
Note  that $g(U)=0$ above and the earlier assumption that $\|U\|<\xi$ together imply that $U\in \M_b$, see~\eqref{eq:manifold}. In view of~\eqref{eq:projTangent} and~\eqref{eq:manifoldGrad}, we also identify the first expression above as $\nabla_{\M_b}f(U)$. We can therefore rewrite \eqref{eq:preGrad} as
\begin{align}
    \nabla_{\M_b} f(U) = 0 \text{ and } g(U)=0 \text{ and } \|U\|<\xi.
\end{align}
That is, in view of Definition~\ref{defn:fosp}, $U$ is an FOSP of \eqref{eq:bm}. This proves the first item of Proposition~\ref{prop:lyapunovPre}. 

To prove the second item in Proposition~\ref{prop:lyapunovPre}, let $U$ be an FOSP of \eqref{eq:bm}. For $\lambda'\in \R^m$, note that the Hessian of the augmented Lagrangian with respect to its first argument is the bilinear operator specified as 
\begin{align}
    \nabla^2_1 L_\g(U,\lambda')
    & = \nabla^2 f(U) - \sum_{i=1}^m (\lambda'_i - \g g_i(U)) \nabla^2 g_i (U) \nonumber\\
    &\qquad  + \g (\Der g(U))^*\circ \Der g(U) 
    \qquad \text{(see \eqref{eq:al})}
    \nonumber\\
    & = \nabla^2 f(U) - \sum_{i=1}^m \lambda'_i \nabla^2 g_i (U) \nonumber\\
    & \quad + \g (\Der g(U))^*\circ \Der g(U), 
    \quad \text{(see \eqref{eq:fosp})}
    \label{eq:HessianALOne}
\end{align}
where $\lambda_i'$ and $g_i(U)$ are the $i^{\text{th}}$ coordinates of the vectors $\lambda'\in \R^m$ and $g(U)$, respectively. In the second line above, we used the fact that $g(U)=0$ for the FOSP $U$. 
For the choice of $\lambda'=\lambda(U,U)$ from~\eqref{eq:lambdaDefn}, we reach
\begin{align}
    \nabla^2_1 L_\g(U,\lambda(U,U)) 
    & = \nabla^2 f(U) - \sum_{i=1}^m \lambda_i(U,U) \nabla^2 g_i (U)  \nonumber\\
    & \qquad + \g (\Der g(U))^*\circ \Der g(U),
    \label{eq:hessianLambdaOurChoice}
\end{align}
where $\lambda_i(U,U)$ is the $i^{\text{th}}$ coordinate of the vector $\lambda(U,U)$. 
 Let $\P_{\T_U\M_b}$ denote the projection onto the tangent space $\T_U\M_b$. Suppose also that $ \nabla^2_1 L_\g(U,\lambda(U,U))[\D,\D] \ge 0$ for every tangent direction $\D\in \T_U\M_b$. It then follows that 
\begin{align}
    0 & \preccurlyeq \P_{\T_U\M_b}\circ \nabla^2_1 L_\g(U,\lambda(U,U)) \circ \P_{\T_U\M_b}  \nonumber\\
    & = \P_{\T_U\M_b}\circ  \Big( \nabla^2 f(U) - \sum_{i=1}^m \lambda_i(U,U) \nabla^2 g_i (U) \nonumber\\
    & \qquad \quad + \g (\Der g(U))^*\circ \Der g(U) \Big) 
    \circ \P_{\T_U\M_b} 
    \qquad \text{(see \eqref{eq:hessianLambdaOurChoice})}
    \nonumber\\
    & = \P_{\T_U\M_b}\circ  \l( \nabla^2 f(U) - \sum_{i=1}^m \lambda_i(U,U) \nabla^2 g_i (U)  \r)  \circ \P_{\T_U\M_b}
    \nonumber\\
    & = \P_{\T_U\M_b}\circ \nabla^2_{\M_b} f(U) \circ  \P_{\T_U\M_b}.
    \qquad \text{(see \eqref{eq:manifoldHessian})}
    \label{eq:manifoldHessianPSD}
\end{align}
To obtain the second identity above, we used \eqref{eq:normalSpace}, and the orthogonality of tangent and normal spaces. In view of Definition~\ref{defn:sosp}, we conclude from~\eqref{eq:manifoldHessianPSD} that $U$ is an SOSP of \eqref{eq:bm}. This proves the second item in Proposition~\ref{prop:lyapunovPre} and completes the proof of Lemma~\ref{prop:lyapunovPre}.

\section{Proof of Lemma \ref{lem:derhGamma}}


Recall from Lemma~\ref{lem:derLambda} that $\lambda(U,U)$ is an analytic function of $U$ on the  set $\{U:\dist(U,\M_b)< \rho\}$, where $\rho$ is the radius of the neighborhood in Assumption~\ref{assumption:key}.   It follows that~$h_\g$ in~\eqref{eq:hGamma}  is also an analytic function of $U$ in the same set. See also~\eqref{eq:fgDefn} and~\eqref{eq:al} to review the notation used in this paragraph. 

To derive an expression for the derivative of $h_\g$ in~\eqref{eq:DerhGammaExplicit}, we calculate the directional derivative of $h_\g$ at $U$ such that $\dist(U, \M_b)<\rho$ and along the direction $\D\in \R^{d\times p}$, i.e.,
\begin{align*}
    & \Der h_\g(U)[\D] \nonumber\\
    & =
    \Der L_\g(U,\lambda(U,U))[\D] \qquad \text{(see \eqref{eq:hGamma})} \nonumber\\
    & = \Der f(U)[\D] - \langle (\Der g(U))^*[\lambda(U,U)-\g g(U)], \D \rangle \nonumber\\
    & \qquad - \langle (\Der \lambda(U,U))^* [g(U)] , \D \rangle \qquad \text{(see \eqref{eq:al})} \nonumber\\
    & = \langle (I_d - \A^*(\lambda(U,U) ))U, \D\rangle+ \frac{\g}{2 }\langle \A^*\l(  \A(UU^\top) -  b \r) U ,\D  \rangle  \nonumber\\
    &\qquad 
    -\frac{1}{2} \langle (\Der \lambda(U,U))^*[\A(UU^\top) - b ], \D \rangle.
    \,\, \text{(see (\ref{eq:fgDefn}), (\ref{eq:adjO}))} \nonumber
\end{align*}
Above, the derivative of $\lambda(U,U)$ with respect to $U$ was computed in Lemma~\ref{lem:derLambda}.
This completes the proof of Lemma~\ref{lem:derhGamma}.

\section{Proof of Proposition \ref{prop:lyapunov}}
The proof is similar to that of Proposition~\ref{prop:lyapunovPre}. Recalling~(\ref{eq:manifold}), we fix $U\in \R^{d\times p}$ such that $\dist(U,\M_b)\le   \rho'$ and $\|U\|<\xi$. Here, $\rho'>0$ is sufficiently small. To be specific, we assume that $\rho'<  \rho$, where $\rho$ is the radius of the neighborhood in Assumption~\ref{assumption:key}.  Our starting point is the expression for $\nabla h_\g(U)$ in~(\ref{eq:DerhGammaExplicit}). 
By comparing this expression with~(\ref{eq:fgDefn}),~(\ref{eq:defnC}) and~\eqref{eq:projTangent}, we can rewrite~\eqref{eq:DerhGammaExplicit} as 
\begin{align}
    &\nabla h_\g(U)  = \l(\mathrm{Id}-(\Der g(U))^*)\circ ((\Der g(U))^*)^\dagger\r) [\nabla f(U)]\nonumber\\
    & \qquad + \g (\Der g(U))^*[g(U)]  - (\Der \lambda(U,U))^*[g(U)].
    \label{eq:beforeProj}
\end{align}
Note that the second term on the right-hand side above is in the range of the operator $(\Der g(U))^*$, whereas the first term on the right-hand side above is orthogonal to this range. 
For brevity, let $\P_U := ((\Der g(U))^*)\circ ((\Der g(U))^*)^\dagger $ denote the orthogonal projection onto the subspace $\range((\Der g(U))^*)$.
After projecting both sides of~\eqref{eq:beforeProj} onto $\range((\Der g(U))^*)$:
\begin{align}
     & \P_U[\nabla h_\g(U)] \nonumber\\
     & = \g (\Der g(U))^*[g(U)]- \l(\P_U \circ (\Der \lambda(U,U))^*)\r) [g(U)]  \nonumber\\
     & =: O_\g(U) [g(U)].
    \label{eq:FOSPhGamma}
\end{align}
Below we make two observations about~\eqref{eq:FOSPhGamma}. 
\begin{enumerate}[leftmargin=*]
    \item It follows from
    the boundedness of $\M_b$ in~\eqref{eq:manifold} that the neighborhood $\{U: \dist(U,\M_b) \le \rho'\}$ is a compact set. In view of Assumption~\ref{assumption:key}, we therefore have that
    \begin{align}
    \min\l\{ \s_{m}(\Der g(U)) : \dist(U,\M_b) \le  \rho' \r\} >0,
    \label{eq:Oleg1}
\end{align}
where $\rho'$ was specified earlier. 

    \item Recall from the first item above that $\{U:\dist(U,\M_b) \le \rho'\}$ is a compact set. Recall also from Lemma~\ref{lem:derLambda} that $\Der \lambda$ is a continuous function on the this set because, earlier in this appendix, we specified that $\rho'<\rho$.
    Consequently,
    \begin{align}
        & \max\l\{ \l\| \P_U  \circ (\Der \lambda(U,U))^*) \r\|: \dist(U,\M_b)\le  \rho'\r\} \le   \nonumber\\
        & \max \l\{ \l\|   (\Der \lambda(U,U))^*) \r\| : \dist(U,\M_b) \le  \rho'\r\}  < \infty,
        \label{eq:Oleg2}
    \end{align}
    where the first inequality above uses the fact that $\P_U$ is an orthogonal projection. 
\end{enumerate}
In view of~\eqref{eq:Oleg1} and~\eqref{eq:Oleg2}, for sufficiently large $\gamma$, we conclude:  
\begin{align}
    \min\l\{ \s_{m}(O_\g(U)) : \dist(U,\M_b) \le  \rho' \r\} >0.
    \label{eq:justEstablished}
\end{align}
In particular, the operator $O_\g(U)$ has a trivial null space for every $U$ such that $\dist(U,\M_b) \le  \rho'$. 

Next, let us consider $\ol{U}$ such that $\dist(U,\M_b)\le  \rho'$ and $\|\ol{U}\|<\xi$. Suppose also that $\ol{U}$ is an FOSP of $h_\g$, i.e., $\nabla h_\g(\ol{U})=0$. It follows from~\eqref{eq:FOSPhGamma} that $O_\g(\ol{U}) [g(\ol{U})]=0$. Since we just established in~\eqref{eq:justEstablished} that  $O_\g(\ol{U}) $ has a trivial null space, $O_\g(\ol{U}) [g(\ol{U})]=0$ in turn implies that $g(\ol{U})=0$.
Combined with the assumption that $\|\ol{U}\|<\xi$, we reach that $\ol{U}\in \M_b$, see~\eqref{eq:fgDefn} and~\eqref{eq:manifold}.  From this point, by following the same steps as in the proof of Proposition~\ref{prop:lyapunovPre}, we  find that $\ol{U}$ is an FOSP of \eqref{eq:bm}. This proves the first item in Proposition~\ref{prop:lyapunov}.  

To prove the second item in Proposition~\ref{prop:lyapunov}, for the same $\ol{U}$ as in the above paragraph, we can use~(\ref{eq:al}) and~(\ref{eq:hGamma}) to calculate the (Euclidean) Hessian of $h_\g$ at $\ol{U}$ as
\begin{align}
    & \nabla^2 h_\g(\ol{U})  = \nabla^2 f(\ol{U}) - \sum_{i=1}^m (\lambda_i(\ol{U})-\g g_i(\ol{U})) \nabla^2 g_i(\ol{U})  \nonumber\\
    & \qquad - (\Der g(\ol{U}))^*\circ \Der\lambda(\ol{U})- (\Der \lambda(\ol{U}))^* \circ  \Der g(\ol{U})
    \nonumber\\
    & \qquad + \g (\Der g(\ol{U}))^*\circ \Der g(\ol{U}) - \sum_{i=1}^m g_i(\ol{U}) \nabla^2 \lambda_i(\ol{U}) 
    \nonumber\\
   & =  \nabla^2 f(\ol{U}) - \sum_{i=1}^m \lambda_i(\ol{U}) \nabla^2 g_i(\ol{U})  \nonumber\\
&\qquad   - (\Der g(\ol{U}))^*\circ \Der\lambda(\ol{U}) - (\Der \lambda(\ol{U}))^* \circ  \Der g(\ol{U})  \nonumber\\
 &  \qquad + \g (\Der g(\ol{U}))^*\circ \Der g(\ol{U}),  \quad \text{(see \eqref{eq:fosp})}
    \label{eq:exphGammaHessian}
\end{align}
where the the second identity above uses the fact that $\ol{U}$ is an FOSP of \eqref{eq:bm}; in particular, $g(\ol{U})=0$ by~\eqref{eq:fosp}.   

Next, let us assume that $\nabla^2 h_\g(\ol{U}) \succcurlyeq 0$. It follows that $\P_{\T_{\ol{U}} \M_{b}} \circ \nabla^2 h_\g(\ol{U}) \circ \P_{\T_{\ol{U}} \M_{{b}}} \succcurlyeq 0$, where $\P_{\T_{\ol{U}} \M_b}$ is the orthogonal projection onto the tangent space of $\M_b$ at $\ol{U}$:
\begin{align}
    0 & \preccurlyeq \P_{\T_{\ol{U}} \M_{b}} \circ \nabla^2 h_\g(\ol{U}) \circ \P_{\T_{\ol{U}} \M_{b}} \nonumber\\
    & = \P_{\T_{\ol{U}} \M_{b}} \circ \Big( 
    \nabla^2 f(\ol{U}) - \sum_{i=1}^m \lambda_i(\ol{U}) \nabla^2 g_i(\ol{U})  \nonumber\\
    & \qquad - (\Der g(\ol{U}))^*\circ \Der\lambda(\ol{U}) - (\Der \lambda(\ol{U}))^* \circ  \Der g(\ol{U}) 
    \nonumber\\
    & \qquad  + \g (\Der g(\ol{U}))^*\circ \Der g(\ol{U})
    \Big) \circ \P_{\T_{\ol{U}} \M_{b}} 
    \qquad \text{(see \eqref{eq:exphGammaHessian})}
    \nonumber\\
    & = \P_{\T_{\ol{U}} \M_{b}} \circ \Big( 
    \nabla^2 f(\ol{U}) - \sum_{i=1}^m \lambda_i(\ol{U}) \nabla^2 g_i(\ol{U}) \Big) \circ \P_{\T_{\ol{U}} \M_{b}} 
    \nonumber\\
    & = \P_{\T_{\ol{U}} \M_{b}} \circ \nabla_{\M_{b}}^2 f({\ol{U}})  \circ \P_{\T_{\ol{U}} \M_{b}},  \qquad \text{(see \eqref{eq:manifoldHessian})}
\end{align}
where, in the second identity above, we used the fact that the tangent space at $\ol{U}\in \M_b$ is orthogonal to $\range((\Der g(\ol{U}))^*)$, see~\eqref{eq:normalSpace}.  
From the last line above, we  conclude that $\ol{U}$ is also an SOSP of \eqref{eq:bm}, see Definition~\ref{defn:sosp}. This proves the second item and  completes the proof of Proposition~\ref{prop:lyapunov}.

\section{Proof of Lemma \ref{lem:rankInvariance}}\label{sec:proofRankInv}

Throughout the remaining proofs, we will often show the dependence on time~$t$ as a subscript. For example, we will write $U_t$ instead of $U(t)$.
In this particular proof, we assume that the radius $\rho_0$ of the neighborhood specified in Lemma~\ref{lem:rankInvariance} is strictly smaller than $\rho$, where $\rho$ is the radius of the neighborhood in Assumption~\ref{assumption:key} and Lemma~\ref{lem:derLambda}. That is, $\rho_0<\rho$. 
Recall from~(\ref{eq:fgDefn}) and~(\ref{eq:DerhGammaExplicit}) that
\begin{align}
    \nabla h_\g(U_t) & = (I_d-\A^*(\lambda(U_t,U_t))) U_t + \g \A^*(g(U_t)) U_t  \nonumber\\
    & \qquad - (\Der \lambda(U_t,U_t))^*[g(U_t)],
    \label{eq:DerhGammaExplicitRecall}
\end{align}
for every $t\in[0,\tau]$. 
Let us also define
\begin{align}
X_t :=U_t U_t^\top \in \R^{d\times d}, \qquad \text{if } t\in [0,\tau]. 
\label{eq:xFlow}
\end{align}
Note that the above flow in $\R^{d\times d}$ satisfies
\begin{align}
   X_0 &= U_0 U_0^\top, \qquad \text{(see \eqref{eq:xFlow})} \label{eq:xFlowDer}\\
    \dot{X}_t & = \dot{U}_t U_t^\top + U_t \dot{U}_t^\top \qquad \text{(see \eqref{eq:xFlow})} \nonumber\\
    &= - \nabla h_\g(U_t) U_t^\top  - U_t (\nabla h_\g(U_t))^\top 
    \, \text{(see \eqref{eq:gradFlowH})}
    \nonumber\\
    & = - (I_d-\A^*(\lambda(U_t,U_t))) X_t  -\g \A^*(g(U_t)) X_t  \nonumber\\
    & + (\Der \lambda(U_t,U_t))^*[g(U_t)] U_t^\top  - X_t (I_d-\A^*(\lambda(U_t,U_t))) \nonumber\\ 
    & \qquad  - \g X_t \A^*(g(U_t))  + U_t \l((\Der \lambda(U_t,U_t))^*[g(U_t)] \r)^\top,\nonumber
\end{align}
where the last identity above uses (\ref{eq:DerhGammaExplicitRecall}) and \eqref{eq:xFlow}. 
In the last identity above, we also used the fact that $\{A_i\}_{i=1}^m$ are symmetric matrices, see~\eqref{eq:measOp}.  
The next technical result establishes that the flow~\eqref{eq:xFlowDer} has an analytic singular value decomposition~(SVD). 
\begin{lem}[{\sc Analytic SVD}]\label{lem:analyticSVDx} Suppose that the assumptions made in Lemma~\ref{lem:rankInvariance} are fulfilled.  Then the flow~\eqref{eq:xFlowDer} has the analytic SVD
\begin{align}
    X_t \overset{\textup{SVD}}{=} V_t S_t V_t^\top,
    \qquad  \text{if }t\in [0,\tau],
    \label{eq:analyticSVDx}
\end{align}
where $V_t\in \R^{d\times d}$ is an orthonormal basis and the diagonal matrix $S_t\in \R^{d\times d}$ contains the singular values of $X_t$ in no particular order. Moreover, $V_t$ and $S_t$ are analytic functions of $t$ on the interval $[0,\tau]$. 
\end{lem}
\begin{proof}
Recall from Lemma~\ref{lem:derhGamma} that  $h_\g(U)$ is an analytic function of $U$ in the set $\{U: \dist(U,\M_b) \le  \rho_0\}$, where~$\rho_0$ was specified in the beginning of this appendix.  By construction,~\eqref{eq:gradFlowH} remains in the $\rho_0$-neighborhood of $\M_b$ on the interval $[0,\tau]$. That is, $\dist(U_t,\M_b) \le \rho_0$ for every $t\in [0,\tau]$.  It therefore follows from Theorem~1.1 in~\cite{2008lectures} that $U_t$ is an analytic function of $t$ on the interval~$[0,\tau]$. Consequently, $X_t = U_t U_t^\top$ is also  an analytic function of $t$ on the interval $[0,\tau]$, see~\eqref{eq:xFlow}. In view of Theorem~1 in~\cite{bunse1991numerical}, it follows that $X_t$  has an analytic SVD, as claimed. 
\end{proof}

By comparing~(\ref{eq:xFlow}) and~(\ref{eq:analyticSVDx}), we record another  simple technical lemma for later use.
\begin{lem}[{\sc Decomposition}]\label{lem:decomposition} Suppose that the assumptions made in Lemma~\ref{lem:rankInvariance} are fulfilled. Then there exist $\{R_t\}_t \subset \R^{d\times p}$ such that
\begin{align}
    U_t = V_t \sqrt{S_t} R_t,
    \qquad  \text{if }t \in [0,\tau],
    \label{eq:uToV}
\end{align}
where the nonzero rows of $R_t\in \R^{d\times p}$ are orthonormal.
\end{lem}
\begin{proof}
Since $V_t$ is an orthonormal basis by Lemma~\ref{lem:analyticSVDx}, we let $U_t= V_t Q_t$ for a matrix $Q_t\in \R^{d\times p}$. It follows from~(\ref{eq:xFlow}) and~(\ref{eq:analyticSVDx}) that $Q_t Q_t^\top = S_t$. 
For notational convenience, suppose that only the first $l$ diagonal entries of the diagonal matrix~$S_t$ are nonzero, for an integer $l\le p$. We let $S_{t,l}\in \R^{l\times l}$ denote the corresponding submatrix of $S_t$.
We can then write that 
\begin{align}
Q_t Q_t^\top = \l[ 
\begin{array}{cc}
    S_{t,l}  & 0_{l\times (d-l)}  \\
   0_{(d-l)\times l}  & 0_{(d-l)\times (d-l)}
\end{array}
\r].
\label{eq:simpleLemma}
\end{align}
It follows from~\eqref{eq:simpleLemma} that 
\begin{equation}
Q_{t,l} Q_{t,l}^\top = S_{t,l}, \qquad Q_{t,l^+}=0,
\label{eq:submatrices}
\end{equation}
where $Q_{t,l}\in \R^{l\times p}$ is the row submatrix of $Q_t$ that corresponds to its first $l$ rows. Similarly, $Q_{t,l^+}\in \R^{(d-l)\times p}$ contains the remaining rows of $Q_t$. It follows from~\eqref{eq:submatrices} that the rows of $Q_{t,l}$ are orthogonal to one another. That is,  $Q_{t,l} = \sqrt{S_{t,l}}R_{t,l} $ for a matrix $R_{t,l}\in \R^{l\times p}$ with orthonormal rows. Here, $\sqrt{S_{t,l}}$ is well-defined because the diagonal matrix $S_{t,l}$ contains  the positive singular values of~$X_t$, see Lemma~\ref{lem:analyticSVDx}. In turn it follows that $Q_t = \sqrt{S_t} R_t$, where we set the remaining row subset of $R_{t}\in \R^{d\times p}$ to zero.  
\end{proof}

In view of Lemma~\ref{lem:analyticSVDx}, we take the derivative with respect to $t$ of both sides of~\eqref{eq:analyticSVDx} to find that 
\begin{align}
    \dot{X}_t = \dot{V}_t S_t V_t^\top + V_t \dot{S}_t V_t^\top + V_t S_t \dot{V}_t^\top, 
    \quad  \text{if } t\in [0,\tau].
\end{align}
By multiplying both sides above by $V_t^\top$ and $V_t$ from left and right, we reach
\begin{align}
    V_t^\top \dot{X}_t V_t & = V_t^\top \dot{V}_t S_t + \dot{S}_t + S_t \dot{V}_t^\top V_t,
    \quad \text{if } t\in [0,\tau].
    \label{eq:preSingEvolveX}
\end{align}
On the right-hand side above, we used the fact that $V_t$ is an orthonormal basis by Lemma~\ref{lem:analyticSVDx}, i.e., $V_t^\top V_t = I_d$. Taking the derivative of both sides of the last identity also yields  
\begin{align}
    \dot{V}_t^\top V_t + V_t^\top \dot{V}_t = 0, 
    \qquad \text{if } t\in [0,\tau].
\end{align}
That is, $V_t^\top \dot{V}_t$ is a skew-symmetric matrix. In particular, both $\dot{V}_t^\top V_t$ and $V_t^\top \dot{V}_t$ are hollow matrices, i.e., both matrices have zero diagonal entries. By taking the diagonal part of both sides of~\eqref{eq:preSingEvolveX}, we thus arrive at 
\begin{align}
    \dot{s}_{t,i} & = v_{t,i}^\top \dot{X}_t v_{t,i},
    \qquad \text{if } t\in [0,\tau],
    \label{eq:singEvolvRawe}
\end{align}
where $s_{t,i}$ is the $i^{\text{th}}$ singular value of $X_t$ and $v_{t,i}\in \R^d$ is the corresponding singular vector.  By substituting above the  expression for $\dot{X}_t$ from~\eqref{eq:xFlowDer}, we find that 
\begin{align}
    \dot{s}_{t,i} 
    & = - 2 s_{t,i} \cdot v_{t,i}^\top \l(I_d - \A^*(\lambda(U_t,U_t)) \r)  v_{t,i} \nonumber\\
    & \qquad - 2 \g s_{t,i}\cdot v_{t,i}^\top \A^*(g(U_t)) v_{t,i} \nonumber\\
    & \qquad + 
    2 \sqrt{s_{t,i}} \cdot v_{t,i}^\top (\Der \lambda(U_t,U_t))^*[g(U_t)] \cdot R_t^\top e_i,
    \label{eq:evolvSing}
\end{align}
for every $t\in [0,\tau]$. Above, we used \eqref{eq:xFlowDer},~(\ref{eq:uToV}) and~(\ref{eq:singEvolvRawe}).
Above, we also used multiple times the fact that $(s_{t,i},v_{t,i})$ is by definition a pair of singular value and its corresponding singular vector for $X_t$. Also, $e_i\in \R^d$ in~\eqref{eq:evolvSing} stands for the $i^{\text{th}}$ canonical vector. That is, only the $i^{\text{th}}$ entry of $e_i$  is nonzero and that entry equals one. 
In view of the evolution of singular values, specified by~\eqref{eq:evolvSing}, it is evident that
\begin{align}
    & \rank(U_t) = \rank(X_t)  \le \rank(X_0) = \rank(U_0), 
\end{align}
for every $t\in [0,\tau]$.
This completes the proof of Lemma~\ref{lem:rankInvariance}. The two identities above follow from~\eqref{eq:xFlow}.

\end{document}